\documentclass[12pt,final]{colt2019} 

\usepackage{xargs}
\usepackage{shortcuts_OPT,bm,stmaryrd}
\usepackage{bbm}
\usepackage[textwidth=1cm,textsize=footnotesize]{todonotes}

\usepackage{mdframed}
\newmdtheoremenv{theo}{Theorem}
\newmdtheoremenv{coro}{Corollary}

\title[Non-asymptotic Analysis of Biased Stochastic Approximation Scheme]{Non-asymptotic Analysis of Biased Stochastic Approximation Scheme}
\usepackage{times}

\coltauthor{\Name{Belhal Karimi} \Email{belhal.karimi@polytechnique.edu}\\
\addr CMAP, \'{E}cole Polytechnique, Palaiseau, France.\\
\Name{Blazej Miasojedow} \Email{bmiasojedow@gmail.com}\\
\addr Faculty of Mathematics, Informatics and Mechanics, University of Warsaw, Poland.\\
\Name{Eric Moulines} \Email{eric.moulines@polytechnique.edu}\\
\addr CMAP, \'{E}cole Polytechnique, Palaiseau, France.\\
\Name{Hoi-To Wai} \Email{htwai@se.cuhk.edu.hk}\\
 \addr Department of SEEM, The Chinese University of Hong Kong, Hong Kong.\thanks{Authors listed in alphabetical order.}} 

\newcommandx\sequence[3][2=,3=]
{\ifthenelse{\equal{#3}{}}{\ensuremath{\{ #1_{#2}\}}}{\ensuremath{\{ #1_{#2}, \eqsp #2 \in #3 \}}}}
\newcommandx\sequencePar[3][2=,3=]
{\ifthenelse{\equal{#3}{}}{\ensuremath{\{ #1({#2})\}}}{\ensuremath{\{ #1({#2}), \eqsp #2 \in #3 \}}}}

\newtheorem{Lemma}{Lemma}
\newtheorem{Prop}{Proposition}

\newtheorem{assumption}{A\!\!}

\newtheorem*{Lemma*}{Lemma}
\newtheorem*{Prop*}{Proposition}

\newcommand{\blue}{\color{black}}

\begin{document}

\maketitle

\begin{abstract}
Stochastic approximation (SA) is a key method used in statistical learning. Recently, its non-asymptotic convergence analysis has been considered in many papers.
However, most of the prior analyses are made under restrictive assumptions such as unbiased gradient estimates and convex objective function, which significantly limit their applications to sophisticated tasks such as online and reinforcement learning. These restrictions are all essentially relaxed in this work.
In particular, we analyze a general SA scheme to minimize a non-convex, smooth objective function. 
We consider update procedure whose drift term depends on a state-dependent Markov chain and the mean field is not necessarily of gradient type, covering approximate second-order method and allowing asymptotic bias for the one-step updates.
We illustrate these settings with the online EM algorithm and the policy-gradient method for average reward maximization in reinforcement learning.
\end{abstract}

\begin{keywords}
biased stochastic approximation, state-dependent Markov chain, non-convex optimization, policy gradient, online expectation-maximization
\end{keywords}


\section{Introduction}
Stochastic Approximation (SA) schemes are sequential (online) methods for finding a zero of a function when only noisy observations of the function values are available. Consider the recursion:\vspace{-.15cm}
\beq \label{eq:sa}
\prm_{n+1} = \prm_n - \gamma_{n+1} \HX{\prm_n}{\State_{n+1}}, \quad n \in \nset \vspace{-.15cm}
\eeq
where $\prm_n \in \Prm \subset \rset^d$ denotes the $n$th iterate, $\gamma_n > 0$ is
the step size and $\HX{\prm_n}{\State_{n+1}}$ is the $n$th \emph{stochastic} update (a.k.a.~drift term) depending on
a random element $\State_{n+1}$ taking its values in a measurable space $\Xset$.
In the simplest setting, $\sequence{\State}[n][\nset]$ is an \iid\ sequence of random vectors and $\HX{\prm_n}{\State_{n+1}}$ is a conditionally \emph{unbiased} estimate of the so-called mean-field $h(\prm_n)$, \ie\
$\CPE{\HX{\prm_n}{X_{n+1}}}{\mcf_n}= h(\prm_n)$ where $\mcf_n$ denotes the filtration generated by the random variables $(\prm_0, \{\State_m \}_{m \leq n})$. In such case, $\be_{n+1} = \HX{\prm_n}{\State_{n+1}} - h(\prm_n)$ is a \emph{martingale difference}. In more sophisticated settings, $\sequence{X}[n][\nset]$ is a \emph{state-dependent} (or controlled) Markov chain, \ie\  for any bounded measurable
function $f : \Xset \rightarrow \rset$,\vspace{-.15cm}
\beq \label{eq:markov} 
\CPE{f ( \State_{n+1})}{ \mcf_n} =  \PX{n} f (\State_n ) = \int f(\state) \PX{n} ( \State_n, \rmd \state ) \eqsp,\vspace{-.15cm}
\eeq
where $\PX{} : \Xset \times \Xsigma \rightarrow \rset_+$ is a Markov kernel such that, for each $\prm \in \Prm$, $\PX{}$ has a unique stationary distribution $\pi_{\prm}$. In such case, the mean field for the SA is defined as: \vspace{-.1cm}
\beq \label{eq:meanfield}  \textstyle
h (\prm) = \int \HX{\prm}{\state} \pi_{\prm} ( \rmd \state ) \eqsp, \vspace{-.1cm}
\eeq
where we have assumed that $\int \|\HX{\prm}{\state} \| \pi_{\prm} ( \rmd \state ) < \infty$.

Throughout this paper, we assume that the mean field $h$ is `related' (to be defined precisely later) to a smooth Lyapunov function $V : \rset^d \rightarrow \rset$, {\blue where $V(\prm) > -\infty$}. The aim of the SA scheme \eqref{eq:sa} is to find a minimizer or stationary point of {\blue the possibly non-convex Lyapunov function} $V$. 

Though more than 60 years old \citep{robbins1985stochastic}, SA is now of renewed interest as it covers a wide range of applications at the heart of many successes with statistical learning.
This includes in particular the stochastic gradient (SG) method and its variants as surveyed in
\citep{bottou1998online,bottou2018optimization}, 
but also in  reinforcement learning \citep{williams1992simple,peters2008natural,barto2018reinforcement}.
Most convergence analyses assume that $\sequence{\prm}[n][\nset]$ is bounded with probability one or visits a prescribed compact set infinitely often. Under such global stability or recurrence conditions [and appropriate regularity conditions on the mean field $h$], the SA sequences might be seen as approximation of the ordinary differential equation  $\dot{\prm} = h(\prm)$. Most results available as of today [see for example \citep{benveniste1990adaptive}, \citep[Chapter~5, Theorem~2.1]{kushner2003stochastic} or \citep{borkar2009stochastic}] have an asymptotic  flavor. The focus is to establish that the stationary point of the sequence  $\sequence{\prm}[n][\nset]$ belongs to a stable attractor of its limiting ODE. 

To gain insights on the difference among statistical learning algorithms, non-asymptotic analysis of SA scheme has been considered only recently.
In particular, SG methods whose mean field is the gradient of the objective function, \ie $h(\prm) = \grd V(\prm)$, are  considered by \cite{moulines2011non} for strongly convex function $V$ and martingale difference noise; see \citep{bottou2018optimization} for a recent survey on the topic.
Extensions to stationary dependent noise have been considered in \citep{duchi2012ergodic,agarwal2013generalization}.
Meanwhile, many machine learning models can lead to non-convex optimization problems. To this end, SG methods for non-convex, smooth objective function $V$ have been first studied in  \citep{ghadimi2013stochastic} with martingale noise (see \cite[Section~4]{bottou2018optimization}), and it was extended in \citep{NIPS2018_8195} to the case where $\sequence{X}[n][\nset]$ is a state-independent Markov chain, \ie the Markov kernel in \eqref{eq:markov} does not depend on $\prm$.

Of course, SA schemes go far beyond SG methods. In fact, in many important applications,  the drift term of the SA is \emph{not} a noisy version of the gradient, \ie the mean field $h$ is not the gradient of $V$. 
Obvious examples include second-order methods, which aim at combatting the adverse effects of high non-linearity and ill-conditioning of the objective function through stochastic quasi-Newton algorithms. Another closely related example is the online Expectation Maximization (EM) algorithm introduced by \citet{cappe2009line} and is further developed in \citep{balakrishnan2017statistical,chen2018stochastic}. 
In many cases, the mean field of the drift term may even be asymptotically biased with the random element $\sequence{X}[n][\nset]$ drawn from a Markov chain with \emph{state-dependent} transition probability. Examples for this situation are common in reinforcement learning such as Q-learning \citep{jaakkola1994convergence}, policy gradient \citep{baxter2001infinite} and temporal difference learning \citep{bhandari2018finite, lakshminarayanan2018linear, dalal2018finite, dalal2017finite}. 
  
Surprisingly  enough, we are not aware of non-asymptotic convergence results of {\blue the general SA \eqref{eq:sa}} comparable to \citep{ghadimi2013stochastic} and \citep[Section~4,5]{bottou2018optimization} when  {\sf (a)} the drift term $\HX{\prm}{x}$ in \eqref{eq:sa} is not the noisy gradient of the objective function $V$ and is potentially biased, and/or {\sf (b)} the sequence $\sequence{X}[n][\nset]$ is a \emph{state-dependent} Markov chain. 
To this end, the main objective of this work is to fill this gap in the literature by establishing non-asymptotic convergence of SA under the above settings. 
Our main assumption is the existence of a smooth function $V$ satisfying  for all $\prm \in \Prm$, $c_0 + c_1 \pscal{ \grd V( \prm )}{ h( \prm) } \geq \| h( \prm ) \|^2$  there exists $c_1 > 0, c_0 \geq 0$; see Section~\ref{sec:sa} and A\ref{ass:SA2}. If $c_0=0$, then $\pscal{\grd V(\prm)}{h(\prm)} > 0$ as soon as $h(\prm) \neq \bm{0}$ in which case $V$ is a Lyapunov function for the ODE $\dot{\prm}= h(\prm)$. Assuming $c_0 > 0$ allows us to consider situations in which the estimate of the mean field is biased, a situation which has been first studied in \cite{tadic2017asymptotic}. 
To summarize, our contributions are two-fold:\vspace{-.2cm}
\begin{enumerate}
\item We provide \emph{non-asymptotic} convergence analysis for \eqref{eq:sa} with a potentially biased mean field $h$ under two cases --- {\sf (Case 1)} $\sequence{\State}[n][\nset]$ is an \iid\ sequence; {\sf (Case 2)} $\sequence{\State}[n][\nset]$ is a \emph{state-dependent} Markov chain. 
For these two cases, we provide non asymptotic bounds such that for all $n \in \nset$,  $\EE [ \| h( \prm_{N} ) \|^2 ] = {\cal O}(c_0 + \log(n) /\sqrt{n})$, for some random index  $N \in \{1,\dots,n\}$ and $c_0 \geq 0$ characterizes the (potential) bias of the mean field $h$.
\item We illustrate our findings by analyzing popular statistical learning algorithms such as the online expectation maximization (EM) algorithm \citep{cappe2009line} and the average-cost policy-gradient method \citep{barto2018reinforcement}. Our findings  provide new insights into the non-asymptotic convergence behavior of these algorithms.
\end{enumerate}
Our theory significantly extends the results reported in \cite[Sections~4,5]{bottou2018optimization} and  \cite[Theorem~2.1]{ghadimi2013stochastic}. When focused on the Markov noise setting, our result is a nontrivial relaxation of \citep{NIPS2018_8195}, which considers Markov noise that is \emph{not state dependent} and the mean field satisfies $h( \prm ) = \grd V( \prm )$; and  of \citep{tadic2017asymptotic} which shows asymptotic convergence of \eqref{eq:sa} under the uniform boundedness assumption on iterates.

\paragraph{Notation}
Let $(\Xset,\Xsigma)$ be a measurable space.
A Markov kernel $R$ on $\Xset \times \Xsigma$ is a mapping $R: \Xset \times \Xsigma \to \ccint{0,1}$ satisfying the following conditions: {\sf (a)} for every $\state \in \Xset$, $R(\state,\cdot): A \mapsto R(\state,\ A)$ is a probability measure on $\Xsigma$ {\sf (b)} for every $\borelA \in \Xsigma$, $R(\cdot,A): \state \mapsto  R(\state,A)$ is a measurable function.
For any probability measure $\lambda$ on $(\Xset,\Xsigma)$, we define $\lambda R$  by $\lambda\Rker(\borelA) = \int_{\Xset} \lambda(\rmd \state) R(\state,\borelA)$.
For all $k\in\nset^*$, we define the Markov kernel $R^k$ recursively by $\Rker^1 = \Rker$ and for all $\state \in\Xset$ and $\borelA\in\Xsigma$,
$  R^{k+1}(\state,\borelA) = \int_{\Xset} R^k(\state, \rmd \state') R(\state', \borelA) $.
A probability measure $\pibar$  is invariant for $R$ if $\pibar R = \pibar$.
$\| \cdot \|$ denotes the standard Euclidean norm (for vectors) or the operator norm (for matrices).


\vspace{-.3cm}
\section{Stochastic Approximation Schemes and Their Convergence} \label{sec:sa}\vspace{-.1cm}
Consider the following assumptions:\vspace{-.1cm}
\begin{assumption} \label{ass:SA2}
For all $\prm \in \Prm$, there exists $c_0 \geq 0, c_1 > 0$ such that
$c_0  + c_1 \pscal{\grd V( \prm )}{ h( \prm ) } \geq \| h( \prm ) \|^2$.\vspace{-.1cm}
\end{assumption}
\begin{assumption} \label{ass:SA2b}
For all $\prm \in \Prm$,
there exists $d_0 \geq 0 ,d_1 > 0$ such that
$d_0 + d_1 \| h ( \prm ) \| \geq \| \grd V( \prm ) \|$.\vspace{-.1cm}
\end{assumption}
\begin{assumption} \label{ass:L}
Lyapunov function $V$ is $L$-smooth. For all $(\prm, \prm') \in \Prm^2$,
$\| \grd V( \prm ) - \grd V( \prm' )  \| \leq L \| \prm - \prm' \|$.\vspace{-.1cm}
\end{assumption}
A\ref{ass:SA2},A\ref{ass:SA2b} assume that the mean field  $h(\prm)$ [cf.~\eqref{eq:markov}]
is indirectly related to the Lyapunov function $V(\prm)$ where it needs not be the same as $\grd V(\prm)$. 
In particular, the constants $c_0, d_0$ characterize  the `bias' between the mean field and the gradient of the Lyapunov function. From an optimization perspective, we note that the Lyapunov function $V$ can be \emph{non-convex} under A\ref{ass:L}.
In light of A\ref{ass:SA2}, A\ref{ass:SA2b}, we study the convergence of the non-negative quantity
$\| h( \prm_n ) \|^2$,
where $\prm_n$ is produced by \eqref{eq:sa}. If $c_0=d_0=0$ in A\ref{ass:SA2},A\ref{ass:SA2b}, then $h(\prm_*) = 0$ implies that $\| \grd V(\prm_*) \| = 0$, \ie the point $\prm_*$ is a stationary point of the deterministic recursion $\bar{\prm}_n= \bar{\prm}_{n} - \gamma_{n+1} h(\bar{\prm}_n)$.
As a convention, for any $\epsilon \geq 0$, we say that $\prm_*$ is an \emph{$\epsilon$-quasi-stationary point} if $\| h( \prm_* ) \|^2 \leq \epsilon$.

As a common step in analyzing SA scheme for smooth but non-convex Lyapunov function (e.g., \citep{ghadimi2013stochastic}), we shall adopt a randomized stopping rule. For any $n \geq 1$, let $N \in \{0,\dots,n\}$ be a discrete random variable (independent of $\sequence{\mcf}[n][\nset]$) with\vspace{-.2cm}
\beq \label{eq:prob_stop} 
\PP( N = \ell ) \eqdef \big( {\textstyle \sum_{k=0}^n \gamma_{k+1}} \big)^{-1} \gamma_{\ell+1} \eqsp,\vspace{-.2cm}
\eeq
where $N$ serves as the terminating iteration for \eqref{eq:sa}. Throughout this paper, we focus on analyzing $\EE[ \| \grd h(\prm_N) \|^2]$ where the expectation is taken over $N$ and the stochastic updates in SA.
We consider two settings for the noise in SA scheme. Define the following noise vector:\vspace{-.2cm}
\beq
\bm{e}_{n+1} \eqdef \HX{\prm_n}{\State_{n+1}} - h ( \prm_n )  \eqsp,\vspace{-.2cm}
\eeq
where $h ( \prm_n )$ was defined in \eqref{eq:meanfield}.
Our settings and convergence results are in order.
\paragraph{{\sf Case 1.}~$\{ {\bm e}_n \}_{n \geq 1}$ is a Martingale Difference Sequence.} We first consider a case similar to the classical SG method analyzed by \citet{ghadimi2013stochastic}. In particular,
\begin{assumption} \label{ass:iid1}
The sequence of noise vectors is a Martingale difference sequence with, for any $n \in \nset$, $\CPE{ {\bm e}_{n+1}}{\mcf_n} = {\bm 0}$, $\CPE{  \| {\bm e}_{n+1}\|^2}{ \mcf_n } \leq \sigma_0^2  + \sigma_1^2 \| h( \prm_n ) \|^2$ with $\sigma_0^2,\sigma_1^2 \in [0,\infty)$.
\end{assumption}
As a concrete example, A\ref{ass:iid1} can be satisfied when $\HX{\prm_n}{\State_{n+1}} = h( \prm_n ) + \State_{n+1}$ where $\State_{n+1}$ is an i.i.d., zero-mean random vector with bounded variance.
We show:
\begin{theo} \label{thm:iid}
Let A\ref{ass:SA2}, A\ref{ass:L}, A\ref{ass:iid1} hold and $\gamma_{n+1} \leq (2 c_1 L (1+\sigma_1^2) )^{-1}$ for all $n \geq 0$. 
We have\vspace{-.2cm}
\beq \label{eq:simple} 
\EE[ \| h( \prm_N ) \|^2 ] \leq \frac{2 c_1 \big( V_{0,n} + \sigma_0^2 L  \sum_{k=0}^n \gamma_{k+1}^2 \big) }{\sum_{k=0}^n \gamma_{k+1}}+ 2 c_0 \eqsp, \vspace{-.15cm}
\eeq
where $N$ is distributed according to \eqref{eq:prob_stop} and we have defined $V_{0,n} \eqdef \EE[ V(\prm_0) - V(\prm_{n+1})]$.
\end{theo}
If we set $\gamma_k = (2 c_1 L (1+\sigma_1^2) \sqrt{k})^{-1}$ for all $k \geq 1$, then the right hand side in \eqref{eq:simple} evaluates to ${\cal O}( c_0+\log n / \sqrt{n} )$ for any $n \geq 1$. Therefore, the SA scheme \eqref{eq:sa} finds an ${\cal O}( c_0+ \log n /\sqrt{n})$ quasi-stationary point within $n$ iterations. 

\paragraph{{\sf Case 2.}~$\{ {\bm e}_n \}_{n \geq 1}$ is State-dependent Markov Noise.}
Next, we consider a general scenario when $\State_{n+1}$ is drawn from a state-dependent Markov process. For any bounded measurable function $\varphi$ and $n \in \nset$, we have $\CPE{\varphi(X_{n+1})}{\mcf_n}= \PX{n} \varphi(\State_n)$, where  $\PX{}{}$ is a Markov kernel on $\Xset \times \Xsigma$. We assume that for each $\prm \in \Prm$, $\PX{}{}$ has a unique stationary distribution $\pi_{\prm}$, \ie\ $\pi_{\prm} \PX{}{}= \pi_{\prm}$. In addition, for each $\prm \in \Prm$, we have $\int \| \HX{\prm}{\state} \| \pi_{\prm}(\rmd \state) < \infty$ and $h(\prm) = \int \HX{\prm}{\state} \pi_{\prm}(\rmd \state)$.
Consider a set of assumptions that are similar to \citep[Section 3]{tadic2017asymptotic}:
\begin{assumption} \label{ass:existence-poisson}
There exists a Borel measurable function $\hat{H}: \Prm \times \Xset \to \Prm$ where for each $\prm \in \Prm$, $\state \in \Xset$,\vspace{-.1cm}
\beq \label{eq:poisson-equation}
\hHX{\prm}{\state} - \PX{}{} \hHX{\prm}{\state} = \HX{\prm}{\state} - h( \prm ) \,. \vspace{-.3cm}
\eeq
\end{assumption}

\begin{assumption} \label{ass:properties-poisson}
There exists $L^{(0)}_{PH} < \infty$ and  $L^{(1)}_{PH} < \infty$ such that, for all $\prm \in \Prm$ and $\state \in \Xset$, one has $\| \hHX{\prm}{\state} \|  \leq L^{(0)}_{PH}, \| \PX{} \hHX{\prm}{\state} \| \leq L_{PH}^{(0)}$.
Moreover, for $(\prm, \prm') \in \Prm^2$,\vspace{-.2cm}
\begin{equation} \label{eq:bound_a2}
\textstyle \sup_{\state \in \Xset} \| \PX{} \hHX{\prm}{\state} - {P}_{\prm'} \hHX{\prm'}{\state}  \|  \leq  L_{PH}^{(1)} \| \prm - \prm' \| \eqsp. \vspace{-.2cm}
\end{equation}
\end{assumption}
\begin{assumption} \label{ass:MC3}
The  stochastic update is bounded, \ie $\sup_{\prm \in \Prm, \state \in \Xset}  \| \HX{\prm}{\state} - h(\prm) \| \leq \sigma$.
\end{assumption}
Assumption A\ref{ass:existence-poisson}  requires that for each $\prm \in \Prm$, the  Poisson equation associated with the Markov kernel $\PX{}{}$ and the function $\HX{\prm}{\cdot}$ has a solution. Assumption A\ref{ass:properties-poisson} implies that for each $\state \in \Xset$, the function $\prm \mapsto \HX{\prm}{\state}$ is Lipshitz and that the Lipshitz constant is uniformly bounded in $\state \in \Xset$. We provide in Appendix~\ref{sec:existence-properties-Poisson} conditions upon which these assumptions hold. 
Lastly, Assumption A\ref{ass:MC3} assumes that the drift terms are bounded uniformly.
Our main result reads as follows:\vspace{-.2cm}
\begin{theo} \label{thm:markov}
Let A\ref{ass:SA2}--A\ref{ass:L}, A\ref{ass:existence-poisson}--A\ref{ass:MC3} hold. Suppose that
the step sizes satisfy\vspace{-.2cm}
\beq \label{eq:stepsize_markov}
\gamma_{n+1} \leq \gamma_n, ~\gamma_n \leq a \gamma_{n+1},~ \gamma_n - \gamma_{n+1} \leq a' \gamma_n^2,~\gamma_1 \leq 0.5 \big( c_1(L+C_h) \big)^{-1} \eqsp, \vspace{-.1cm}
\eeq
for some $a, a' > 0$ and all $n \geq 0$. We have \vspace{-.1cm}
\beq
\label{eq:notsimple}
 \EE [ h( \prm_N ) \|^2 ] \leq
 \frac{ 2c_1 \big( V_{0,n} + C_{0,n} + \big( \sigma^2 L + C_\gamma \big) \sum_{k=0}^{n}
 \gamma_{k+1}^2 \big) }{ \sum_{k=0}^n \gamma_{k+1} } + 2 c_0 \eqsp,\vspace{-.1cm}
\eeq
where $N$ is distributed according to \eqref{eq:prob_stop}, $V_{0,n} \eqdef \EE[ V(\prm_0) - V(\prm_{n+1})]$, and the constants are:\vspace{-.2cm}
\begin{align}
\label{eq:definition-Ch}
C_h      &\eqdef \big( L_{PH}^{(1)} (d_0+\frac{d_1}{2}(a+1) + a d_1 \sigma) + L_{PH}^{(0)} \big(  L + d_1\{1 + a'\} \big) \big) \eqsp, \\
\label{eq:defintion-Cgamma}
C_\gamma &\eqdef  L_{PH}^{(1)} (d_0+d_0 \sigma + d_1 \sigma) + L L_{PH}^{(0)} (1+\sigma) \eqsp, \\
\label{eq:defintion-C0n}
C_{0,n}  &\eqdef L_{PH}^{(0)} \big( (1+d_0)(\gamma_1 - \gamma_{n+1}) + d_0 (\gamma_1 + \gamma_{n+1})+2 d_1 \big) \eqsp.
\end{align}
\end{theo}\vspace{-.2cm}
Similar to the case with Martingale difference noise, if we set $\gamma_k = (2 c_1 L (1+C_h) \sqrt{k})^{-1}$ for all $k \geq 1$, then the step size satisfies \eqref{eq:stepsize_markov} with $a=\sqrt{2}$ and $a' = \frac{\sqrt{2}-1}{\sqrt{2}} (2c_1L(1+C_h))$, and the right hand side in \eqref{eq:notsimple} evaluates to ${\cal O}( c_0 + \log n / \sqrt{n} )$ for any $n \geq 1$. We obtain a similar convergence rate as in Theorem~\ref{thm:iid}.
In fact, if we consider a special case  when for all $\prm \in \Prm$ and $\state \in \Xset$, $\PX{}(\state,\cdot)= \pi_{\prm}(\cdot)$, we have $L_{PH}^{(0)} = L_{PH}^{(1)} = 0$. The constants evaluates to $C_h = C_\gamma = C_{0,n} = 0$ and our Theorem~\ref{thm:markov} can be reduced into Theorem~\ref{thm:iid}.
We remark that Theorem~\ref{thm:markov} cannot be treated as a strict generalization of Theorem~\ref{thm:iid} as A\ref{ass:iid1} does not imply the uniform boundedness A\ref{ass:MC3}.
Our analysis [cf.~Lemma~\ref{lem:prod}] relies on a new decomposition of the error terms, which controls the growth of $\EE[ \|h( \prm_n )\|^2 ]$ without explicitly assuming that $\{ \prm_n \}_{n \geq 0}$ is bounded. 

In Appendix~\ref{app:lowerbd}, we provide a lower bound on the rate of SA scheme \eqref{eq:sa}, \eqref{eq:prob_stop} such that $\EE[ \| h( \prm_N ) \|^2 ] = \Omega( \log n / \sqrt{n} )$. This shows that our analysis in Theorem~\ref{thm:iid}, \ref{thm:markov} is tight. \vspace{.2cm} 

\noindent \textbf{Related Studies}~~~~Non-asymptotic analysis of biased SA schemes can be found in the literature on temporal difference (TD) learning \citep{bhandari2018finite, lakshminarayanan2018linear, dalal2018finite, dalal2017finite}, which analyzed a special case of linear SA. Their assumptions can essentially be satisfied by our A\ref{ass:SA2}--A\ref{ass:L} with $V(\prm) = \| \prm - \prm^\star \|_\Phi^2$, e.g., \citep[Lemma 3]{bhandari2018finite}  shows that the TD learning has a mean field which satisfies A\ref{ass:SA2}. Furthermore, the above mentioned analysis requires a strongly convex Lyapunov function, which is not needed in our results.

For {\sf Case 1}, our results generalizes \cite[Theorem~2.1]{ghadimi2013stochastic} by accounting for biased SA updates. In fact we recover the latter result with $h(\prm) = \grd V(\prm)$,  A\ref{ass:SA2} [$c_0 = 0, c_1=1$].

For {\sf Case 2}, our assumptions A\ref{ass:SA2}--A\ref{ass:L}, A\ref{ass:existence-poisson}--A\ref{ass:MC3} are similar to  \citep[Section 3]{tadic2017asymptotic}. The   exception is A\ref{ass:MC3} which is used in place of the assumption  $\sup_{n \in \NN} \| \prm_n \| < \infty$ in  \citep{tadic2017asymptotic}. We note that the two conditions are neither stronger nor weaker than the other.\vspace{-.1cm}


\subsection{Convergence Analysis}
The detailed proofs in this section are in Appendix~\ref{app:sa}.
To simplify notations, we denote $h_n \eqdef \| h( \prm_n ) \|^2$ from now on.
We first describe an intermediate result that holds under just A\ref{ass:SA2}, A\ref{ass:L}:
\begin{Lemma} \label{prop:conv}
Let A\ref{ass:SA2}, A\ref{ass:L} hold. It holds for all $n \geq 1$ that:
\beq \label{eq:fin1}
\begin{split}
& \textstyle \sum_{k=0}^n \frac{\gamma_{k+1}}{c_1} \big( 1 - c_1 L \gamma_{k+1} \big) h_k
\\[.0cm] & \textstyle \leq V( \prm_0 ) - V(\prm_{n+1} ) + L  \sum_{k=0}^{n}
 \gamma_{k+1}^2 \| \bm{e}_{k+1} \|^2 + \sum_{k=0}^n \gamma_{k+1} \big( \frac{c_0}{c_1} - \pscal{ \grd V( \prm_k ) }{ \bm{e}_{k+1} } \big).
\end{split}
\eeq
\end{Lemma}
Having established Lemma~\ref{prop:conv}, our main convergence results can be obtained as follows.\vspace{-.1cm}

\paragraph{Proof of Theorem~\ref{thm:iid}} 
With Martingale difference noise, the expected value of $ \pscal{ \grd V( \prm_k ) }{ {\bm e}_{k+1} }$ is zero when conditioned on $\mcf_k$. Therefore, taking total expectation on both sides of \eqref{eq:fin1} yields:\vspace{-.1cm}
\beq \notag
\begin{split} \textstyle
 \sum_{k=0}^n \frac{\gamma_{k+1}}{c_1} \big( 1 - c_1 L \gamma_{k+1} \big) \EE[ h_k ] & \textstyle \leq
V_{0,n} + L  \sum_{k=0}^{n} \big( \gamma_{k+1}^2 \EE[\| \bm{e}_{k+1} \|^2] + \gamma_{k+1} \frac{c_0}{c_1} \big) \\
 & \leq \textstyle
V_{0,n} + L  \sigma_0^2 \sum_{k=0}^{n} \gamma_{k+1}^2  + L \sigma_1^2 \sum_{k=0}^n \gamma_{k+1} \EE[ h_k ] )  + \gamma_{k+1} \frac{c_0}{c_1} \big) \eqsp,
 \end{split}
\eeq
where the last inequality is due to A\ref{ass:iid1}. Rearranging terms yields:\vspace{-.1cm}
\beq \textstyle
 \sum_{k=0}^n \frac{\gamma_{k+1}}{c_1} \big( 1 - c_1 L  (1+\sigma_1^2) \gamma_{k+1} \big) \EE[ h_k ]
 \leq V_{0,n} + \sigma_0^2 L  \sum_{k=0}^{n} \gamma_{k+1}^2 + \frac{c_0}{c_1} \sum_{k=0}^n \gamma_{k+1} \eqsp.\vspace{-.1cm}
\eeq
Consequently, using \eqref{eq:prob_stop} and noting that 
$1 - c_1 L  (1+\sigma_1^2) \gamma_{n+1} \geq \frac{1}{2}$,
we obtain\vspace{-.1cm}
\beq \label{eq:simple_case_expect}
\EE[ h_N ] = \sum_{n'=0}^n \frac{ \gamma_{n' +1}  \EE[ h_{n'}] } {\sum_{k=0}^n  \gamma_{k+1} }
\leq \frac{2c_1 \big( V_{0,n} + \sigma_0^2 L  \sum_{k=0}^{n} \gamma_{k+1}^2 \big)}{ \sum_{k=0}^n \gamma_{k+1}} + 2c_0 \eqsp. \vspace{-.1cm}
\eeq

\paragraph{Proof of Theorem~\ref{thm:markov}} In the case with state-dependent Markovian noise. Under A\ref{ass:MC3}, one has\vspace{-.1cm}
\beq \label{eq:sq_term} \textstyle
\sum_{k=0}^{n}
 \gamma_{k+1}^2 \EE[ \| \bm{e}_{k+1} \|^2]  \leq \sum_{k=0}^{n} \gamma_{k+1}^2  \sigma^2 \eqsp.\vspace{-.1cm}
\eeq
Unlike  in Theorem~\ref{thm:iid},
the expected value of the inner product $ \pscal{ \grd V( \prm_k ) }{ {\bm e}_{k+1} }$ is non-zero in general. Fortunately, as we show next in Lemma~\ref{lem:prod}, this issue can be mitigated.
\begin{Lemma} \label{lem:prod}
Let A\ref{ass:SA2}--A\ref{ass:L},A\ref{ass:existence-poisson}--A\ref{ass:MC3} hold and the step sizes satisfy \eqref{eq:stepsize_markov}. It holds:\vspace{-.1cm}
\beq \label{eq:mc_fin} \textstyle
\EE \big[ - \sum_{k=0}^n \gamma_{k+1} \pscal{ \grd V( \prm_k ) }{ \bm{e}_{k+1} } \big] \leq
C_h \sum_{k=0}^n \gamma_{k+1}^2  \EE[\| h( \prm_k ) \|^2] + C_{\gamma} \sum_{k=0}^n \gamma_{k+1}^2 + C_{0,n} \eqsp,\vspace{-.1cm}
\eeq
where $C_h$, $C_{\gamma}$ and $C_{0,n}$ are defined in \eqref{eq:definition-Ch}, \eqref{eq:defintion-Cgamma}, \eqref{eq:defintion-C0n}.
\end{Lemma}
Finally, to prove the theorem, we combine Lemma~\ref{prop:conv}, \eqref{eq:sq_term} and Lemma~\ref{lem:prod} to obtain:\vspace{-.1cm}
\beq
\begin{split}
& \textstyle \sum_{k=0}^n \frac{\gamma_{k+1}}{c_1} \big( 1 - c_1 L \gamma_{k+1} \big) \EE [h_k] \\
& \textstyle \leq V_{0,n} + C_{0,n} + \Big( \sigma^2 L + C_\gamma \Big) \sum_{k=0}^{n}
 \gamma_{k+1}^2 + C_h \sum_{k=0}^n \gamma_{k+1}^2 \EE [ h_k ] + \frac{c_0}{c_1} \sum_{k=0}^n \gamma_{k+1} \eqsp.\\[-.3cm]
\end{split}
\eeq
Repeating a similar argument as in \eqref{eq:simple_case_expect} using the distribution \eqref{eq:prob_stop} shows the desired bound \eqref{eq:notsimple}.


\vspace{-.2cm}
\section{Applications} \label{sec:application}\vspace{-.1cm}
We present several applications and provide new non-asymptotic convergence rate for them.\vspace{-.2cm}
\subsection{Regularized Online Expectation Maximization}
Expectation-Maximization (EM) \citep{dempster1977maximum}  is a powerful tool for learning latent variable models, which can be inefficient due to the high storage cost. 
 This has motivated the development of online version of the EM which makes it possible to estimate the parameters of
latent variables model without storing the data; the online EM algorithm analyzed below was introduced in \citep{cappe2009line} and later developed by many authors: see for example \citep{chen2018stochastic} and the references therein. The online EM algorithm sticks closely to the
principles of the batch-mode EM algorithm. Each iteration of the online EM algorithm is decomposed into two steps, where the first one is a stochastic approximation version of
the E-step aimed at incorporating the information brought by the newly available observation, and, the second step consists in the maximization program that appears in the M-step of the traditional EM algorithm.

The latent variable statistical model postulates the existence of a latent variable $\State$ distributed under $f(\state;\param)$ where $\{ f(\state; \param); \param \in \Param\}$ is a parametric family of probability density functions and $\Param$ is an open convex subset of $\rset^d$. The observation $Y \in \Yset$ is a deterministic function of $\State$. We denote by $g(y;\param)$ the (observed) likelihood function. The notations $\EE_{\param}[\cdot]$ and $\CPE[\param]{\cdot}{Y}$ are used to denote the expectation and conditional expectation under the statistical model $\{ f(\state; \param); \param \in \Param\}$. We denote by $\pi$ the probability density function of the observation $Y$:  the model might be misspecified, that is, the "true" distribution of the observations may not belong to the family $\{ g(y;\param), \param \in \Param \}$. The notations $\PE_\pi$ is used below to denote the expectation under the actual distribution of the observations.    Let $\Sset$ be a convex open subset of $\rset^m$ and $S: \Xset \to \Sset$ be a measurable function.
We assume that the complete data-likelihood function belongs to the curved exponential family\vspace{-.1cm}
\begin{equation}
\label{eq:exp_fam}
f(\state;\param)= h(\state) \exp\left( \pscal{ S(\state) }{ \phi( \param ) } - \psi ( \param ) \right) \eqsp,\vspace{-.1cm}
\end{equation}
where $\psi: \Param \to \rset$ is  twice differentiable and convex and $\phi: \Param \to \Sset \subset \rset^m$ is concave and differentiable.
In this setting, $S$ is the complete data sufficient statistics.
For any $\param \in \Param$ and $y \in \Yset$, we assume that the conditional expectation\vspace{-.1cm}
\begin{equation}
\label{eq:estep}
\overline{\bm{s}}(y;\param)= \CPE[\param]{S(X)}{Y=y}\vspace{-.1cm}
\end{equation}
is well-defined and belongs to  $\Sset$. For any $\bm{s} \in \Sset$, we consider the penalized negated complete data log-likelihood:\vspace{-.15cm}
\begin{equation}
\label{eq:negated:pernalized}
\ell( \bm{s}; \param ) \eqdef \psi( \param ) + \Pen ( \param ) - \pscal{ {\bm s} }{ \phi( \param )} \eqsp,\vspace{-.1cm}
\end{equation}
where $\Pen: \Param \mapsto \rset$ is a penalization term assumed to be twice differentiable. This penalty term is used to enforce constraints on the
estimated parameter.
If $\kappa: \Param \to \rset^m$ is a differentiable function, we denote by $\jacob{\kappa}{\param}{\param'} \in \rset^{m \times d}$ the Jacobian of the map $\kappa$ with respect to $\param$ at $\param'$.
Consider:
\begin{assumption} \label{ass:reg}
For all $\bm{s} \in \Sset$, the function $\param \mapsto \ell( \bm{s}; \param )$ admits a unique global minimum $\mstep{\bm{s}}$ in the interior of $\Param$, characterized by\vspace{-.4cm}
\beq \label{eq:a9_cond}
\grd \psi( \mstep{\bm s} ) + \grd \Pen( \mstep{\bm s} ) - \jacob{\phi}{\param}{ \mstep{\bm s} }^\top {\bm s} = {\bm 0} \eqsp.\vspace{-.15cm}
\eeq
In addition,  for any $\bm{s} \in \Sset$, $\jacob{\phi}{\param}{ \mstep{\bm s} }$ is invertible and  the map $\bm{s} \mapsto \mstep{\bm{s}}$ is differentiable on $\Sset$.
\end{assumption}
The \emph{regularized} version of the online EM (ro-EM) method is an iterative procedure which alternatively updates an estimate of the sufficient statistics and the estimated parameters as:
\beq \label{eq:oem}
\hat{\bm s}_{n+1} = \hat{\bm s}_n + \gamma_{n+1} \big( \overline{\bm s}( Y_{n+1}; \hat{\param}_n ) - \hat{\bm s}_n \big),~~\hat{\param}_{n+1} = \overline{\param} ( \hat{\bm s}_{n+1} )\eqsp.
\eeq
In the following, we show that our \emph{non-asymptotic} convergence result holds for  the ro-EM.
We establish convergence of the online method to a stationary point of the Lyapunov function defined as a regularized Kullback-Leibler (KL) divergence between $\pi$ and $g_\param$. Precisely, we set
\beq \label{eq:lya}
V( {\bm s}) \eqdef \KL{\pi}{g( \cdot; \mstep{\bm s} \big)} + \Pen( \mstep{\bm s}  ), \quad \KL{\pi}{ g( \cdot; \param )} \eqdef \EE_{ \pi } \left[ \log(\pi ( Y))/g(Y; \theta) \right] \eqsp.
\eeq
We establish a few key results that relate the ro-EM method to an SA scheme seeking for a stationary point of $V({\bm s})$.
Denote by  $\mcf_n$  the filtration generated by the random variables $\{  \hat{\bm s}_{0}, Y_{k} \}_{ {k} \leq n}$. From \eqref{eq:oem} we can identify the drift term and its mean field respectively as\vspace{-.1cm}
\beq \label{eq:hs}
\begin{split}
& H_{ \hat{\bm s}_n } (Y_{n+1}) = \hat{\bm s}_n -  \overline{\bm s}(Y_{n+1}; \mstep{ \hat{\bm s}_n } ) \eqsp, \\[.1cm]
& h( \hat{\bm s}_n ) = \EE_{ \pi } \big[ H_{ \hat{\bm s}_n } (Y_{n+1}) | {\cal F}_n \big] = \hat{\bm s}_n - \EE_\pi \big[  \overline{\bm s}(Y_{n+1}; \mstep{ \hat{\bm s}_n } ) \big] \eqsp.
\end{split}\vspace{-.1cm}
\eeq
and ${\bm e}_{n+1} \eqdef H_{ \hat{\bm s}_n } (Y_{n+1}) - h( \hat{\bm s}_n )$.
Define by $\hess{\ell}{\param}$  the Hessian of the function $\ell$ with respect to $\param$. Our results are summarized by the following propositions, which proofs can be found in Appendix~\ref{app:roem}:
\begin{Prop} \label{prop:hstar}
Assume A\ref{ass:reg}. The following holds:\vspace{-.1cm}
\begin{itemize}
\item If $h( {\bm s}^\star ) = {\bm 0}$ for some $\bm{s}^\star \in \Sset$, then $\grd_{\param} \KL{\pi}{g_{\param^\star}}  + \grd_{ \param } \Pen( \param^\star )  = {\bm 0}$ with  $\param^\star \eqdef \mstep {\bm{s}^\star}$.\vspace{-.1cm}
\item If $\grd_{\param} \KL{\pi}{g_{\param^\star}}  + \grd_{ \param } \Pen( \param^\star )  = {\bm 0}$ for some $\param^\star \in \Param$ then $\bm{s}^\star= \PE_\pi[S(Y,\param^\star)]$.
\end{itemize}
\end{Prop}\vspace{-.3cm}
\begin{Prop} \label{prop:asat}
Assume A\ref{ass:reg}. We have $
\grd_{ {\bm s} } V( {\bm s} ) = \jacob{ \phi }{ \param }{ \mstep{\bm s}} \big( \hess{\ell}{\param}( {\bm s}; \param )  \big)^{-1} \jacob{ \phi }{ \param }{ \mstep{\bm s}}^\top \!~ h ( {\bm s})$ for ${\bm s} \in \Sset$.
\end{Prop}\vspace{-.1cm}
Proposition \ref{prop:hstar} relates the root(s) of the mean field $h( {\bm s} )$ to the stationary condition of the regularized KL divergence. Moreover, if $\lambda_{\sf min} \big( \jacob{\phi}{\param}{\mstep{\bm s}} \big( \hess{\ell}{\param} ( {\bm s}; \mstep{\bm s})   \big)^{-1} \jacob{\phi}{\param}{\mstep{\bm s}}^\top \big) \geq \upsilon > 0$  for all ${\bm s} \in \Sset$, then 
Proposition \ref{prop:asat} shows that the mean field of the stochastic update in \eqref{eq:hs} satisfies A\ref{ass:SA2} with $c_0 = 0$ and $c_1 = 1 / \upsilon$. If we assume that the Lyapunov function in \eqref{eq:lya}, and the stochastic update in \eqref{eq:hs} satisfy the assumptions in Case 1 [\ie A\ref{ass:iid1}], then these results show that Theorem~\ref{thm:iid} applies.
To further illustrate the above principles, we look at an example with Gaussian mixture model (GMM).

\paragraph{Example:~GMM Inference}\hspace{-.2cm}
Consider the inference problem of a mixture of $M$ Gaussian distributions, each with a unit variance from an observation stream $Y_1,Y_2,\ldots$. The likelihood is:
\beq \label{eq:like} \textstyle
g(y ; \param) \propto   \Big(1-\sum_{m=1}^{M-1} \omega_m\Big) \exp\left(- \frac{(y - \mu_M)^2}{2} \right)  + \sum_{m=1}^{M-1}
\omega_m \exp\left(- \frac{(y - \mu_m)^2}{2} \right) \eqsp.
\eeq
The parameters are denoted by $\param \eqdef (\omega_1, \dots,\omega_{M-1}, \mu_1, \dots, \mu_{M-1}, \mu_M) \in \Cset$ where the parameter set is defined as $\Cset = \Delta_{M-1} \times \rset^M$ with $\Delta_{M-1} \eqdef \{ ( \omega_1, \cdots, \omega_{M-1}) \in \rset^{M-1}, \omega_m \geq 0,\sum_{m=1}^{M-1} \omega_m \leq 1 \} $.
To apply the ro-EM method, we augment the $n$th data $Y_n$ with the latent variable $Z_n \in \{1,\dots,M\}$. The log likelihood of the complete data tuple is
\beq \label{eq:comp_like} \textstyle
{\cal L}( {\bm x}; \param) =
\indiacc{z=M} \left[ \log(1 - \sum_{m=1}^{M-1}\omega_m) - \frac{(y - \mu_M)^2}{2} \right] + \sum_{m=1}^{M-1} \indiacc{z=m} \left[ \log(\omega_m)-\frac{(y - \mu_m)^2}{2} \right] \eqsp.
\eeq
The above can be written in the standard curved exponential family form \eqref{eq:exp_fam}. In particular, we partition the sufficient statistics as
$S( {\bm x} ) = ( S^{(1)}( {\bm x} )^\top , S^{(2)}( {\bm x} )^\top, S^{(3)}({\bm x}) )^\top \in \rset^{M-1} \times \rset^{M-1} \times \rset$, and partition $\phi( \param ) = ( \phi^{(1)}( \param )^\top ,\phi^{(2)}( \param )^\top,\phi^{(3)}( \param ) )^\top \in \rset^{M-1} \times \rset^{M-1} \times \rset$. Using the fact that $\indiacc{z=M} = 1 - \sum_{m=1}^{M-1} \indiacc{z=m}$,   \eqref{eq:comp_like} can be expressed in the standard form as \eqref{eq:exp_fam} with \vspace{-.3cm}
\beq \label{eq:gmm_exp}
\begin{split}
& \textstyle s_m^{(1)} = \indiacc{z=m}, \quad
   \phi_m^{(1)}(\param) =   \left\{\log(\omega_m) -\frac{\mu_m^2}{2}\right\} - \left\{\log(1 - \sum_{j=1}^{M-1} \omega_j) - \frac{\mu_M^2}{2}\right\} \eqsp,\\
  & s_m^{(2)} =   \indiacc{z=m} y, \quad \phi^{(2)}_m(\param) =  {\mu_m},\quad m = 1,\dots,M-1, \quad s^{(3)} = y, \quad \phi^{(3)}(\param) = \mu_M \eqsp, 
\end{split}
\eeq
and $\psi(\param) =   - \left\{\log(1 - \sum_{j=1}^{M-1} \omega_j) - \frac{\mu_M^2}{2 \sigma^2}\right\}$.

We apply the ro-EM method to the above model. Following the partition of sufficient statistics and parameters in the above, we define $\hat{\bm s}_n = ( (\hat{\bm s}_n^{(1)})^\top,(\hat{\bm s}_n^{(2)})^\top, \hat{s}^{(3)} )^\top \in \rset^{M-1} \times \rset^{M-1} \times \rset$, and  $\hat{\param}_n = ( \hat{\bm \omega}_n^\top, \hat{\bm \mu}_n^\top, \hat{\mu}_M )^\top \in \rset^{M-1} \times \rset^{M-1} \times \rset$. Also, define the conditional expected value:
\beq \label{eq:cexp}
\widetilde{\omega}_m ( Y_{n+1} ; \hat{\param}_n ) \eqdef \EE_{\hat{\param}_n}[ \indiacc{z=m} | Y= Y_{n+1} ]
= \frac{ \hat{\omega}_{m,n} \!~ {\rm exp}(-\frac{1}{2}( Y_{n+1} - \hat{\mu}_{m,n} )^2) }{  \sum_{j=1}^{M}{ \hat{\omega}_{j,n} \!~ \exp(-\frac{1}{2}( Y_{n+1} - \hat{\mu}_{j,n} )^2)} } \eqsp.
\eeq
With the above notations, the {\sf E-step}'s update in \eqref{eq:estep} can be described with
\beq \label{eq:estep_gmm}
\overline{s} ( Y_{n+1} ; \hat{\param}_n )
= \left(
\begin{array}{c}
 \big( \widetilde{\omega}_1 ( Y_{n+1}; \hat{\param}_n ), \dots, \widetilde{\omega}_{M-1} ( Y_{n+1} ; \hat{\param}_n ) \big)^\top \\
  \big( Y_{n+1} \widetilde{\omega}_1 ( Y_{n+1}; \hat{\param}_n ), \dots, Y_{n+1}  \widetilde{\omega}_{M-1} ( Y_{n+1} ; \hat{\param}_n ) \big)^\top \\
  Y_{n+1}
\end{array}
\right)
= \left(
\begin{array}{c}
\overline{\bm s}_n^{(1)} \\
\overline{\bm s}_n^{(2)} \\
\overline{s}_n^{(3)}
\end{array}
\right) \eqsp.
\eeq
For the {\sf M}-step, let $\epsilon > 0$ be a user designed parameter, we consider
the following regularizer:
\beq \textstyle \label{eq:regu}
\Pen( \param ) =  \epsilon \sum_{m=1}^M \big\{ \mu_m^2 / 2 - \log ( \omega_m ) \big\} - \epsilon \log \big( 1 - \sum_{m=1}^{M-1} \omega_m \big) \eqsp,
\eeq
For any ${\bm s}$ with ${\bm s}^{(1)} \geq {\bm 0}$,
it can be shown that the regularized {\sf M-step} in \eqref{eq:oem} evaluates to
\beq \label{eq:mstep_gmm}
\overline{\param} ( {\bm s} )
= \left(
\begin{array}{c}
( 1+\epsilon M )^{-1} \big( {s}_1^{(1)} + \epsilon, \dots,  {s}_{M-1}^{(1)} + \epsilon \big)^\top \vspace{.2cm}\\
 \big( ({s}_1^{(1)} + \epsilon)^{-1} {s}_1^{(2)}  , \dots, ({s}_{M-1}^{(1)} + \epsilon)^{-1} {s}_{M-1}^{(2)}  \big)^\top \vspace{.2cm} \\
  \big(1 - \sum_{m=1}^{M-1}s_m^{(1)} + \epsilon\big)^{-1} \big( s^{(3)} - \sum_{m=1}^{M-1} s_m^{(2)} \big)
\end{array}
\right)
= \left(
\begin{array}{c}
\overline{\bm{\omega}} ( {\bm s}) \\
\overline{\bm{\mu}} ( {\bm s}) \\
\overline{\mu}_M ( {\bm s})
\end{array}
\right) \eqsp.
\eeq
Note that, as opposed to an unregularized solution (\ie with $\epsilon=0$), the regularized solution is numerically stable as it avoids issues such as division by zero.

To analyze the convergence of ro-EM, we verify that \eqref{eq:oem}, \eqref{eq:estep_gmm}, \eqref{eq:mstep_gmm} yield a special case of an SA scheme on $\hat{\bm s}_n$ which satisfies A\ref{ass:SA2}, A\ref{ass:L}, A\ref{ass:iid1}. Assume the following on the observations $\{Y_n\}_{n \geq 0}$
\begin{assumption} \label{ass:bdd}
Each observed sample $Y_n$ is drawn i.i.d.~and they are bounded as $|Y_n| \leq \overline{Y}$ for any $n \geq 0$.
\end{assumption}
The ro-EM method can be initialized by setting $\hat{\bm s}_1 = ( {\bm 0} , {\bm 0}, 0)^\top$ and begun with the {\sf M}-step.
Note that under A\ref{ass:bdd}, the sufficient statistics $\hat{\bm s}_n$ lie in the compact set $\Sset = \Delta_{M-1} \times [- \overline{Y} , \overline{Y} ]^{M}$  for all $n \geq 1$, where $\Delta_{M-1} \eqdef \{ s_1,\dots,s_{M-1} : s_m \geq 0,~\sum_{m=1}^{M-1} s_m \leq 1\}$.
We observe the following propositions that are proven in Appendix~\ref{app:roem}:
\begin{Prop} \label{prop:variance}
Under A\ref{ass:bdd},
it holds that $\EE[ \| \overline{s} ( Y_{n+1} ; \hat{\param}_n ) - \hat{\bm s}_n \|^2 | {\cal F}_n ] \leq 2M \overline{Y}^2$ for all $n \geq 0$ \eqsp.
\end{Prop}
\begin{Prop} \label{prop:big}
Under A\ref{ass:bdd} and the regularizer \eqref{eq:regu} set with $\epsilon > 0$, then for all $({\bm s},{\bm s}') \in \Sset^2$, there exists positive constants $\upsilon, \Upsilon, \Psi$ such that:
\beq \label{eq:big_res}
\pscal{ \grd V( {\bm s} ) }{ h( {\bm s} ) } \geq \upsilon \!~ \| h( {\bm s} ) \|^2,~~
\| \grd V( {\bm s} ) - \grd V({\bm s}') \| \leq \Psi \| {\bm s} - {\bm s}' \| \eqsp.
\eeq
\end{Prop}
The above propositions show that the ro-EM method applied to GMM is a special case of the SA scheme with Martingale difference noise, for which A\ref{ass:SA2} [with $c_0 = 0$, $c_1 = \upsilon^{-1}$], and A\ref{ass:L} [with $L= \Psi$], A\ref{ass:iid1} [with $\sigma_0^2 = 2M \overline{Y}^2$, $\sigma_1^2 = 0$] are satisfied. As such, applying Theorem~\ref{thm:iid} shows that 

\begin{coro} \label{cor:roem}
Under A\ref{ass:bdd} and set $\gamma_k = (2 c_1 L (1+\sigma_1^2) \sqrt{k})^{-1}$.
For any $n \in \NN$, let $N \in \{0,...,n\}$ be an independent discrete r.v.~distributed according to \eqref{eq:prob_stop}. 
The ro-EM method for GMM \eqref{eq:oem}, \eqref{eq:estep_gmm}, \eqref{eq:mstep_gmm} finds $\hat{\bm s}_N$ such that
\beq
\EE [ \| \grd V( \hat{\bm s}_N ) \|^2 ] = {\cal O}( \log n / \sqrt{n} )
\eeq
where $V(\cdot)$ is defined in \eqref{eq:lya}. The expectation is taken \wrt $N$ and the observation law $\pi$.
\end{coro}

\paragraph{Related Studies} 
Convergence analysis for the EM method in batch mode has been the focus of  the classical work by \citet{dempster1977maximum, wu1983convergence}, in which  asymptotic convergence has been established; also see the recent work by \citet{wang2015high,xu2016global}.
Several work has studied the convergence of stochastic EM  with \emph{fixed data}, e.g., \citet{mairal2015incremental} studied the asymptotic convergence to a stationary point, \citet{chen2018stochastic} studied the local linear convergence of a variance reduced method by assuming that the iterates are bounded.
On the other hand, the online EM method considered here, where a fresh sample is drawn at each iteration, has only been considered by a few work. Particularly, \citet{cappe2009line} showed the asymptotic convergence of the online EM method to a stationary point; \citet{balakrishnan2017statistical} analyzed non-asymptotic convergence for a variant of online EM method which requires a-priori the initial radius $\| \param_0 - \param^\star\|$, where $\param^\star$ is the optimal parameter.
To our best knowledge, the rate results in Corollary~\ref{cor:roem} is new.

\subsection{Policy Gradient for Average Reward over Infinite Horizon} \label{sec:policy_grad}
There has been a growing interest in policy-gradient methods for  model-free planning in  Markov decision process; see \citep{barto2018reinforcement} and the references therein. Consider a finite  Markov Decision Process (MDP)  $( \Sset, \Aset, \Reward, \transMDP )$, where $\Sset$ is a finite set of  spaces (state-space), $\Aset$ is a finite set of action (action-space), $\Reward: \Sset \times \Aset \to [0,\Reward_{\max}]$ is a reward function and $\transMDP$ is the transition model, \ie given an action $a \in \Aset$,  $\transMDP^a =\{ \transMDP^a_{s,s'}\}$  is a matrix, $\transMDP_{s,s'}^a$ is the probability of transiting from the $s$th state to the $s'$th state upon taking action $a$.
The agent's decision is characterized by a parametric family of policies  $\{ \bm{\Policy}_{\prm} \}_{\prm \in \Prm}$: $\Policy_{\prm}( a ; s )$  which is the  probability of taking action $a$ when the current state is $s$ 
 (a semi-column is used to distinguish the random variables from parameters of the distribution).
The state-action sequence $\{ (S_t,A_t) \}_{t \geq 1}$ forms an MC with the transition matrix:
\beq \label{eq:rprm} \
Q_{\prm} ( (s,a) ; (s',a') ) \eqdef \Policy_{\prm}( a'; s') \transMDP_{s,s'}^a \eqsp,
\eeq
where the above corresponds to the $(s,a)$th row, $(s',a')$th column of the matrix $\bm{Q}_{\prm}$, and it denotes the transition probability from $(s,a)\in \Sset \times \Aset$ to $(s',a') \in \Sset \times \Aset$.  

We assume that for each $\prm \in \Prm$, the policy $\Policy_{\prm}$ is ergodic, \ie\ $\bm{Q}_{\prm}$ has a unique stationary distribution $\invarMDP$. Under this assumption, the \emph{average reward} (or undiscounted reward) is given by
\beq \label{eq:mdp_avg} \textstyle
J( \prm ) \eqdef \sum_{s,a} \invarMDP(s,a) \Reward(s,a) \eqsp.
\eeq
The goal of the agent is to find a policy that maximizes the average reward over the class $\{ \bm{\Policy}_{\prm} \}_{\prm \in \Prm}$.
It can be verified \citep{barto2018reinforcement} that  the gradient is evaluated by the limit:
\beq \label{eq:policy_exact} \textstyle
\grd J( \prm ) = \lim_{T \rightarrow \infty} \EE_{\prm} \big[ \Reward( S_T, A_T ) \sum_{i=0}^{T-1} \grd  \log \Policy_{\prm} ( A_{T-i};S_{T-i}) \big] \eqsp.
\eeq
To approximate \eqref{eq:policy_exact} with a numerically stable estimator, \citep{baxter2001infinite} proposed the following gradient estimator. Let $\lambda \in [0,1)$ be a discount factor and $T$ be sufficiently large, one has
\beq \label{eq:policy_grad_approx} \textstyle
\widehat{\grd}_T J( \prm ) \eqdef \Reward( S_T, A_T ) \sum_{i=0}^{T-1} \lambda^{i} \!~ \grd  \log \Policy_{\prm} ( A_{T-i}; S_{T-i}) \approx \grd J(\prm) \eqsp,
\eeq
where $(S_1,A_1,\dots,S_T,A_T)$ is a realization of state-action sequence generated by the policy $\Policy_{\prm}$.
This gradient estimator is \emph{biased} and its bias is of order $O(1-\lambda)$ as the discount factor $\lambda \uparrow 1$.
The approximation above leads to the following policy gradient method \citep{baxter2001infinite}:\vspace{-.1cm}
\begin{subequations} \label{eq:policy_grad}
\begin{align}
G_{n+1} & = \lambda G_n + \grd \log \Policy_{\prm_n}( A_{n+1}; S_{n+1}) \eqsp, \label{eq:pg_G} \\
\prm_{n+1} & = \prm_n + \gamma_{n+1} G_{n+1} \Reward( S_{n+1}, A_{n+1} ) \eqsp. \label{eq:policy_grad_b} 
\end{align}
\end{subequations}
We focus on a linear parameterization of the policy in the exponential family (or soft-max):
\beq \textstyle \label{eq:softmax_policy}
\Policy_{\prm} ( a ; s) = \big\{ \sum_{a' \in \Aset} \exp \big( \pscal{\prm}{ {\bm x}(s,a') - {\bm x}(s,a) } \big) \big\}^{-1} \eqsp,
\eeq
where ${\bm x}(s,a) \in \rset^d$ is a known feature vector.
We make the following assumptions:
\begin{assumption} \label{ass:bdd_feat}
For all $s \in \Sset$, $a \in \Aset$, the feature vector ${\bm x}(s,a)$ and reward $\Reward(s,a)$ are bounded with $\|{\bm x}(s,a)\| \leq \overline{b}, | \Reward(s,a) | \leq \Reward_{\max}$.
\end{assumption}
\begin{assumption} \label{ass:markov}
For all $\prm \in \Prm$, the MC $\{ (S_t,A_t) \}_{t \geq 1}$, as governed by the transition matrix $\bm{Q}_{\prm}$ [cf.~\eqref{eq:rprm}], is uniformly geometrically ergodic: there exists $\rho \in [0,1)$, $K_R < \infty$  such that, for all $n \geq 0$,
\beq \label{eq:uniform_ergodic}
\| \bm{Q}_{\prm}^n - {\bm 1} {\bm{\upsilon}}_{\prm}^\top \| \leq \rho^n K_R \eqsp,
\eeq
where ${\bm{\upsilon}}_{\prm} \in \rset_+^{|\Sset||\Aset|}$ is the stationary distribution of $\{ (S_t,A_t) \}_{t \geq 1}$. Moreover, 
there exists $L_Q , L_{\upsilon} < \infty$ such that for any $(\prm, \prm') \in \Prm^2$,
\beq \label{eq:lipschitz_stat}
\| \bm{\upsilon}_{\prm} - \bm{\upsilon}_{\prm'} \| \leq L_Q \| \prm - \prm' \| \eqsp, \quad
\|  \jacob{\bm{\upsilon}_{\prm}}{\prm}{\prm} - \jacob{\bm{\upsilon}_{\prm}}{\prm}{\prm'} \| \leq L_{\upsilon} \| \prm - \prm' \| \eqsp,
\eeq
where $\jacob{\bm{\upsilon}_{\prm}}{\prm}{\prm}$ denotes the Jacobian of $\bm{\upsilon}_{\prm}$ \wrt $\prm$.
\end{assumption}
Both A\ref{ass:bdd_feat} and A\ref{ass:markov} are regularity conditions on the MDP model that essentially hold as we focus on the finite state/action spaces setting. 
Under the uniform ergodicity assumption \eqref{eq:uniform_ergodic}, the Lipschitz continuity conditions \eqref{eq:lipschitz_stat} can be implied using \citep{fort2011convergence, tadic2017asymptotic}.

Our task is to verify that the policy gradient method \eqref{eq:policy_grad} is an SA scheme with state-dependent Markovian noise [cf.~Case 2 in Section~\ref{sec:sa}]. To this end, we denote the joint state of this SA scheme as $\State_{n} = ( S_n, A_n, G_n )\in {\sf X} \eqdef \Sset \times \Aset \times \rset^d $, and notice that $\{ \State_{n} \}_{n \geq 1}$ is a Markov chain.
Adopting the same notation as in Section~\ref{sec:sa}, the drift term and its mean field can be written as
\beq \label{eq:SA_policy}
\HX{\prm_n}{\State_{n+1}} = G_{n+1} \Reward(S_{n+1},A_{n+1}) \quad \text{with} \quad
h( \prm ) = \lim_{T \rightarrow \infty} \EE_{\tau_T \sim \Policy_{\prm},~S_1 \sim \overline{\Policy}_{\prm} } \big[ \widehat{\grd}_T J( \prm ) \big] \eqsp,
\eeq
where $\widehat{\grd}_T J( \prm )$ is defined in \eqref{eq:policy_grad_approx}.
Moreover, we let $\PX{}: {\sf X} \times {\cal X} \rightarrow \rset_+$ to  be the Markov kernel associated with the MC $\{ \State_{n} \}_{n \geq 1}$.
Observe that
\begin{Prop} \label{prop:bdd_grad}
Under A\ref{ass:bdd_feat}, it holds for any $(\prm, \prm') \in \Prm^2$, $(s,a) \in \Sset \times \Aset$,
\beq \label{eq:bdd_policy_grad}
\| \grd \log \Policy_{\prm}( a ; s) \| \leq 2 \overline{b},~~\| \grd \log \Policy_{\prm}( a ; s) - \grd \log \Policy_{\prm'}( a ; s) \| \leq 8 \overline{b}^2 \| \prm - \prm' \| \eqsp.
\eeq
\end{Prop}
Using the recursive update of \eqref{eq:pg_G}, we show that
\beq
\label{eq:bdd_Gn}
\| G_n \| = \| \lambda G_{n-1} + \grd \log \Policy_{\prm} ( A_n ; S_n) \| \leq \lambda \| G_{n-1} \| + 2 \overline{b} = {\cal O}( 2 \overline{b} \|G_0\| / (1-\lambda) ) \eqsp,
\eeq
for any $n \geq 1$, which then implies that
the stochastic update $\HX{\prm_n}{\State_{n+1}}$ in \eqref{eq:policy_grad} is bounded since the reward is bounded using A\ref{ass:bdd_feat}.
The above proposition also implies that $h( \prm )$ is bounded for all $\prm \in \Prm$. Therefore, the assumption A\ref{ass:MC3} is satisfied.

Next, with a slight abuse of notation, we shall consider the compact state space $\Xset = \Sset \times \Aset \times {\sf G}$, with ${\sf G} = \{ g \in \rset^d : \| g \| \leq C_0 \overline{b} / (1-\lambda) \}$ and $C_0 \in [1,\infty)$, and analyze the policy gradient algorithm accordingly where  $\{ X_{n+1} \}_{n \geq 0}$ is in $\Xset$.
%
Consider the following propositions whose proofs are adapted from \citep{fort2011convergence, tadic2017asymptotic} and can be found in Appendix~\ref{app:pg}:
\begin{Prop} \label{prop:a4a5}
Under A\ref{ass:bdd_feat}, A\ref{ass:markov}, the following function is well-defined:
\beq \label{eq:poisson_def_policy} \textstyle
\hHX{\prm}{x} = \sum_{t=0}^\infty \big\{ \PPX{t}{} \HX{\prm}{x} - h(\prm) \big\} \eqsp,
\eeq
and satisfies Eq.~\eqref{eq:poisson-equation}.
For all $\state \in \Xset$, $(\prm, \prm') \in \Prm^2$, there exists constants $L_{PH}^{(0)}$, $L_{PH}^{(1)}$ where
\beq \label{eq:lipschitz_policy}
\max\{ \| \PX{} \hHX{\prm}{x} \|, \| \hHX{\prm}{x} \| \} \leq L_{PH}^{(0)},~~ \| \PX{} \hHX{\prm}{\state} - {P}_{\prm'} \hHX{\prm'}{\state} \Big\|  \leq  L_{PH}^{(1)} \| \prm - \prm' \| \eqsp.
\end{equation}
Moreover, the constants are in the order of $L_{PH}^{(0)} = {\cal O}( \frac{1}{1- \max\{\rho,\lambda\}} )$, $L_{PH}^{(1)} = {\cal O}( \frac{1}{1- \max\{\rho,\lambda\}} )$.
\end{Prop}

\begin{Prop} \label{prop:a1a2}
Under A\ref{ass:bdd_feat}, A\ref{ass:markov}, the gradient $\grd J(\prm)$ is $\Upsilon$-Lipschitz continuous, where we defined $\Upsilon \eqdef \Reward_{\max} \!~ |{\cal S}| |{\cal A}|$. Moreover, for any $\prm \in \Prm$ and let $\Gamma \eqdef 2 \overline{b} \!~ \Reward_{\max}  K_R \frac{ 1 }{ (1 - \rho)^2}$, it holds that
\beq
 (1 - \lambda)^2 \Gamma^2 + 2 \pscal{ \grd J(\prm) }{ h( \prm )} \geq \| h( \prm ) \|^2,~\| \grd J( \prm ) \| \leq \| h( \prm ) \| + (1-\lambda) \Gamma \eqsp.
\eeq
\end{Prop}
Proposition~\ref{prop:a4a5} verifies A\ref{ass:existence-poisson} and A\ref{ass:properties-poisson} for the policy gradient algorithm, while Proposition~\ref{prop:a1a2} implies A\ref{ass:SA2} [with $c_0 = (1-\lambda)^2 \Gamma^2$, $c_1 = 2$], A\ref{ass:SA2b} [with $d_0 = (1-\lambda) \Gamma$, $d_1 = 1$], A\ref{ass:L} [with $L = \Upsilon$]. As such, applying Theorem~\ref{thm:markov} gives
\begin{coro} \label{cor:pg}
Under A\ref{ass:bdd_feat}, A\ref{ass:markov} and set $\gamma_k = (2 c_1 L (1+C_h) \sqrt{k})^{-1}$. For any $n \in \NN$, let $N \in \{0,...,n\}$ be an independent discrete r.v.~distributed according to \eqref{eq:prob_stop}, the policy gradient algorithm \eqref{eq:policy_grad} finds a policy, $\prm_N$, with 
\beq
\EE \big[ \| \grd J( \prm_N ) \|^2 \big] = {\cal O} \Big( (1-\lambda)^2 \Gamma^2 + \big(  (1 - \max\{\rho, \lambda\} \sqrt{n} ) \big)^{-1} \log n \Big),
\eeq
where $J(\cdot)$ is defined in \eqref{eq:mdp_avg}. The expectation is taken \wrt $N$ and action-state pairs $(A_n,S_n)$. 
\end{coro}

\noindent Our result highlights the bias-variance tradeoff through the parameter $\lambda \in (0,1)$. In fact, $\lambda \uparrow 1$ reduces the bias but increases the number of iterations required to reach a quasi stationary point. 


\paragraph{Related Studies} The convergence of policy gradient method is typically studied for the \emph{episodic} setting where the goal is to maximize the total reward over a \emph{finite horizon}. The {\sf REINFORCE} algorithm \citep{williams1992simple} has been analyzed as an SG method with \emph{unbiased} gradient estimate in \citep{sutton2000policy}, which proved an asymptotic convergence condition. A recent work \citep{pmlr-v80-papini18a} combined the variance reduction technique with the {\sf REINFORCE} algorithm.

The \emph{infinite horizon} setting is more challenging. To our best knowledge, the first asymptotically convergent policy gradient method is the actor-critic algorithm by \citet{konda2003onactor} which is extended to off-policy learning in \citep{degris2012off}. The analysis are based on the theory of two time-scales SA, which relies on controlling the ratio between the two set of step sizes used \citep{borkar1997stochastic}. On the other hand, the algorithm which we have studied was a direct policy gradient method proposed by \citet{baxter2001infinite}, whose asymptotic convergence was proven only recently by \citet{tadic2017asymptotic}. 
In comparison, our Corollary~\ref{cor:pg} provides the first non-asymptotic convergence for the policy gradient method. Of related interest, it is worthwhile to mention that \citep{fazel2018global,abbasi2018regret} have studied the global convergence for average reward maximization under the linear quadratic regulator setting where the state transition can be characterized by a linear dynamics and the reward is a quadratic function.\vspace{-.2cm}



\section{Conclusion}

In this paper, we analyze under mild assumptions a general SA scheme with either \emph{zero-mean} [cf.~Case 1] or \emph{state-dependent/controlled Markovian} [cf.~Case 2] noise. 
We establish a novel \emph{non-asymptotic} convergence analysis of this procedure without assuming convexity of the Lyapunov function.
In both cases, our results highlight a convergence rate of order ${\cal O}(\log(n)/\sqrt{n})$ under conservative assumptions.
We verify our findings on two applications of growing interest: the online EM for learning  an exponential family distribution (e.g., Gaussian Mixture Model) and the policy gradient method for maximizing an average reward.



\newpage
\section*{Acknowledgement}
HTW's work is supported by the CUHK Direct Grant \#4055113. The authors would like to thank the anonymous reviewers for valuable feedback. 

\bibliography{references}

\newpage
\appendix

\section{Analysis of the SA Schemes} \label{app:sa}

\subsection{Proof of Lemma~\ref{prop:conv}}
\begin{Lemma*}
Assume A\ref{ass:SA2}, A\ref{ass:L}. Then, for all $n \geq 1$, it holds that:
\beq
\begin{split}
& \textstyle \sum_{k=0}^n \frac{\gamma_{k+1}}{c_1} \big( 1 - c_1 L \gamma_{k+1} \big) h_k
\\[.1cm] & \textstyle \leq V( \prm_0 ) - V(\prm_{n+1} ) + L  \sum_{k=0}^{n}
 \gamma_{k+1}^2 \| \bm{e}_{k+1} \|^2 + \sum_{k=0}^n \gamma_{k+1} \big( c_1^{-1} c_0 - \pscal{ \grd V( \prm_k ) }{ \bm{e}_{k+1} } \big) \eqsp.
\end{split}
\eeq
\end{Lemma*}
\begin{proof}
As the Lyapunov function $V(\prm)$ is $L$ smooth [cf.~A\ref{ass:L}], we obtain:
\beq
\begin{split}
V( \prm_{k+1} ) & \leq V( \prm_k ) - \gamma_{k+1} \pscal{ \grd V( \prm_k )}{  \HX{\prm_k}{\State_{k+1}} } + \frac{L \gamma_{k+1}^2}{2} \|  \HX{\prm_k}{\State_{k+1}} \|^2 \\
& \leq V( \prm_k ) - \gamma_{k+1} \pscal{ \grd V( \prm_k ) }{ h( \prm_k) + \bm{e}_{k+1} } + L \gamma_{k+1}^2 \big( \| h(\prm_{k}) \|^2 + \| \bm{e}_{k+1} \|^2 \big)\eqsp.
\end{split}
\eeq
The above implies that
\beq \label{eq:also_imply}
\begin{split}
\gamma_{k+1} \pscal{ \grd V( \prm_k ) }{ h( \prm_k ) } &
\leq V( \prm_k ) - V( \prm_{k+1} ) - \gamma_{k+1} \pscal{ \grd V( \prm_k ) }{ \bm{e}_{k+1} }\\
& + L \gamma_{k+1}^2 \big( \| h(\prm_{k}) \|^2 + \| \bm{e}_{k+1} \|^2 \big)\eqsp.
\end{split}
\eeq
Using A\ref{ass:SA2}, $\pscal{ \grd V( \prm_k ) }{ h( \prm_k ) } \geq \frac{1}{c_1} ( h_k - c_0)$ and rearranging terms, we obtain
\beq \label{eq:also_imply}
\begin{split}
\frac{ \gamma_{k+1} }{c_1} \big( 1 - c_1 L \gamma_{k+1} \big) h_k &
\leq V( \prm_k ) - V( \prm_{k+1} ) - \gamma_{k+1} \pscal{ \grd V( \prm_k )}{ \bm{e}_{k+1} }\\
& + L \gamma_{k+1}^2 \| \bm{e}_{k+1} \|^2 + \frac{c_0}{c_1} \gamma_{k+1}\eqsp.
\end{split}
\eeq
Summing up both sides from $k=0$ to $k=n$ gives the conclusion \eqref{eq:fin1}.
\end{proof}

\subsection{Proof of Lemma~\ref{lem:prod}}

\begin{Lemma*}
Assume A\ref{ass:SA2}--A\ref{ass:L},A\ref{ass:existence-poisson}--A\ref{ass:MC3}  and the step sizes satisfy \eqref{eq:stepsize_markov}. Then:
\beq
\EE\left[ - \sum_{k=0}^n \gamma_{k+1} \pscal{ \grd V( \prm_k ) }{ \bm{e}_{k+1} } \right] \leq
C_h \sum_{k=0}^n \gamma_{k+1}^2  \EE[\| h( \prm_k ) \|^2] + C_{\gamma} \sum_{k=0}^n \gamma_{k+1}^2 + C_{0,n}\eqsp,
\eeq
where $C_h$, $C_{\gamma}$ and $C_{0,n}$ are defined in \eqref{eq:definition-Ch}, \eqref{eq:defintion-Cgamma}, \eqref{eq:defintion-C0n}.
\end{Lemma*}

\begin{proof}
Under A\ref{ass:existence-poisson}, A\ref{ass:MC3},  for any $\prm \in \Prm$ there exists a bounded, measurable function $ \state  \to \hHX{\prm}{\state}$ such that the Poisson equation holds:
\beq
{\bm e}_{n+1} = \HX{\prm_n}{\State_{n+1}} - h( \prm_n ) = \hHX{\prm_n}{\State_{n+1}} - \PX{n} \hHX{\prm_n}{\State_{n+1}} \eqsp.
\eeq
The inner product on the left hand side of \eqref{eq:mc_fin} can thus be decomposed as
\begin{align}
\EE\left[ -\sum_{k=0}^n \gamma_{k+1} \pscal{ \grd V( \prm_k ) }{ {\bm e}_{k+1}} \right] = \EE[ A_1 + A_2 + A_3 + A_4 + A_5 ] \eqsp,
\end{align}
with
\begin{align*}
A_1 &\eqdef -\sum_{k=1}^n \gamma_{k+1} \pscal{ \grd V( \prm_k ) }{ \hHX{\prm_k}{\State_{k+1}} - \PX{k} \hHX{\prm_k}{\State_k} } \eqsp, \\
A_2 &\eqdef -\sum_{k=1}^n \gamma_{k+1} \pscal{ \grd V( \prm_k ) }{ \PX{k} \hHX{\prm_k}{\State_k} - \PX{k-1} \hHX{\prm_{k-1}}{\State_k} } \eqsp, \\
A_3 &\eqdef -\sum_{k=1}^n \gamma_{k+1} \pscal{ \grd V( \prm_k ) - \grd V(\prm_{k-1} ) }{ \PX{k-1} \hHX{\prm_{k-1}}{\State_k} } \eqsp, \\
A_4 &\eqdef -\sum_{k=1}^n \big( \gamma_{k+1} - \gamma_k \big) \pscal{ \grd V( \prm_{k-1} )}{\PX{k-1} \hHX{\prm_{k-1}}{\State_k} } \eqsp, \\
A_5 &\eqdef - \gamma_1 \pscal{\grd V(\prm_0)}{\hHX{\prm_0}{\State_1}} + \gamma_{n+1} \pscal{ \grd V( \prm_n) }{ \PX{n} \hHX{\prm_n}{\State_{n+1}} }\eqsp.
\end{align*}

For $A_1$, we note that $\hHX{\prm_k}{\State_{k+1}} - \PX{k} \hHX{\prm_k}{\State_k}$ is a martingale
difference sequence [cf.~\eqref{eq:markov}] and therefore we have $\EE [ A_1 ] = 0$
by taking the total expectation.

For $A_2$, applying the Cauchy-Schwarz inequality and \eqref{eq:bound_a2}, we have
\beq
\begin{split}
A_2 & \leq L_{PH}^{(1)} \sum_{k=1}^n  \gamma_{k+1} \| \grd V( \prm_k ) \| \| \prm_k - \prm_{k-1} \| \\
& = L_{PH}^{(1)} \sum_{k=1}^n  \gamma_{k+1} \gamma_k \| \grd V( \prm_k ) \| \| \HX{\prm_{k-1}}{\State_k} \| \\
& \overset{(a)}{\leq} L_{PH}^{(1)} \sum_{k=1}^n \gamma_{k+1} \gamma_k \big( d_0 + d_1 \| h(\prm_k) \| \big) \big( \| h( \prm_{k-1} ) \| + \sigma  \big) \\
& \overset{(b)}{\leq}  L_{PH}^{(1)} \sum_{k=1}^n \gamma_{k+1}\gamma_k
\Big( d_0 \sigma + d_0 \| h(\prm_{k-1} ) \| + d_1 \sigma \| h( \prm_k ) \| + d_1 \| h(\prm_k) \| \| h( \prm_{k-1}) \|  \Big) \eqsp,
\end{split}
\eeq
where (a) is due to A\ref{ass:SA2b} on the norm of $\grd V(\prm_k)$ and A\ref{ass:MC3}
on the norm of ${\bm e}_k$, (b) is obtained by expanding the scalar product.
Using the inequality $\| h( \prm_n ) \| \leq 1 + \| h( \prm_n ) \|^2$
and $2 \| h(\prm_k) \| \| h( \prm_{k-1}) \| \leq \| h( \prm_k) \|^2 + \| h( \prm_{k-1} \|^2$, we obtain:
\beq
A_2  \leq
L_{PH}^{(1)} \left( (d_0 + d_0\sigma + d_1\sigma) \sum_{k=1}^n \gamma_k^2
+ \big( d_0 + \frac{d_1}{2} + a d_1 \sigma + \frac{a d_1}{2} \big)
\sum_{k=0}^n \gamma_{k+1}^2 \| h( \prm_{k} ) \|^2 \right).
\eeq
For $A_3$, we obtain
\beq
\begin{split}
A_3 & \overset{(a)}{\leq} L \sum_{k=1}^n \gamma_{k+1} \gamma_k \| \HX{\prm_{k-1}}{\State_k} \| \|\PX{k-1} \hHX{\prm_{k-1}}{\State_k} \| \\
& \overset{(b)}{\leq} L L_{PH}^{(0)} \sum_{k=1}^n \gamma_{k+1} \gamma_k \big( \| h(\prm_{k-1}) \| + \sigma \big) \\
& \leq L L_{PH}^{(0)} \left( (1+\sigma) \sum_{k=1}^n \gamma_k^2 + \sum_{k=1}^n \gamma_k^2 \| h( \prm_{k-1} ) \|^2 \right) \eqsp,
\end{split}
\eeq
where (a) uses A\ref{ass:L}, (b) uses
$\HX{\prm_{k-1}}{\State_k} = h(\prm_{k-1}) + {\bm e}_k$ and A\ref{ass:properties-poisson}.

For $A_4$, we have
\beq
\begin{split}
A_4 & \leq \sum_{k=1}^n | \gamma_{k+1} - \gamma_k | \big( d_0 + d_1 \| h( \prm_{k-1} ) \| \big) \| \PX{k-1} \hHX{\prm_{k-1}}{\State_k} \| \\
& \overset{(a)}{\leq} L_{PH}^{(0)} \left( (d_0+1) \sum_{k=1}^n | \gamma_{k+1} - \gamma_k | + d_1 \sum_{k=1}^n | \gamma_{k+1} - \gamma_k | \| h( \prm_{k-1} ) \|^2 \right) \\
& \overset{(b)}{=} L_{PH}^{(0)} \left( (d_0+1) \big( \gamma_1 - \gamma_{n+1} \big)  + a' d_1 \sum_{k=1}^n \gamma_k^2 \| h( \prm_{k-1} ) \|^2 \right) \eqsp,
\end{split}
\eeq
where (a) is again an application of A\ref{ass:properties-poisson}, and (b) uses the assumptions on step size $\gamma_{k+1} \leq \gamma_k$, $\gamma_k - \gamma_{k+1} \leq a' \gamma_k^2$.
Finally, for $A_5$, we obtain
\beq
\begin{split}
A_5 & \overset{(a)}{\leq} \gamma_1 \big( d_0 + d_1 \| h( \prm_0 ) \| \big) L_{PH}^{(0)} + \gamma_{n+1} \big( d_0 + d_1 \| h( \prm_n) \| \big) L_{PH}^{(0)} \\
& \overset{(b)}{\leq}
L_{PH}^{(0)} \Big( d_0 \{ \gamma_1  + \gamma_{n+1} \} + 2 d_1 + d_1\{ \gamma^2_1 \| h(\eta_0) \|^2 + \gamma_{n+1}^2 \| h(\eta_n) \|^2 \} \Big) \\
& \leq L_{PH}^{(0)} \Big( d_0 \{ \gamma_1  + \gamma_{n+1} \} + 2 d_1 + d_1 \sum_{k=0}^n \gamma_{k+1}^2 \| h( \prm_k )\|^2 \Big) \eqsp,
\end{split}
\eeq
where (a) is an application of A\ref{ass:SA2b} and A\ref{ass:properties-poisson}, and (b) uses  $a \leq 1 + a^2$.
Gathering the relevant terms and taking expectations conclude the proof of this lemma.
\end{proof}

\subsection{Lower bound on the rate of SA scheme} \label{app:lowerbd}
We provide a lower bound on $\EE[ \| h( \prm_N ) \|^2]$ with the SA scheme \eqref{eq:sa} and \eqref{eq:prob_stop}:\vspace{-.1cm}
\begin{Lemma} \label{lem:lowerbd}
Consider the SA scheme \eqref{eq:sa} with $h(\prm) = \grd V(\prm)$. 
There exists a Lyapunov function $V(\prm)$ satisfying A\ref{ass:L} and a noise sequence $\{ {\bm e}_{n} \}_{n \geq 1}$ satisfying A\ref{ass:iid1}-A\ref{ass:MC3} such that for any $n \geq 1$, \vspace{-.2cm}
\beq \label{eq:lowerbd}
\EE[ \| h( \prm_N ) \|^2 ] \geq \frac{ \EE \big[ V( \prm_0 ) - V( \prm_{n+1} ) \big] + C_{\sf lb} \sum_{k=0}^n \gamma_{k+1}^2 }{ \sum_{k=0}^n \gamma_{k+1} }\vspace{-.1cm}
\eeq
where $N$ is distributed according to \eqref{eq:prob_stop}, 
and $C_{\sf lb} > 0$ is some constant independent of $n$.\vspace{-.1cm}
\end{Lemma}

For large $n$, setting $\gamma_k = c / \sqrt{k}$ minimizes the right hand side of \eqref{eq:lowerbd}, yielding $\EE[ \| h( \prm_N ) \|^2 ] = \Omega( \log(n) / \sqrt{n})$. The considered SA scheme satisfies assumptions A\ref{ass:SA2}-A\ref{ass:MC3}, and the lower bound \eqref{eq:lowerbd} matches the upper bounds in Theorem~\ref{thm:iid} \& \ref{thm:markov} (when $c_0=0$). The upper bounds are therefore tight.

We remark that our proof in Appendix~\ref{app:lowerbd} uses the construction with a strongly convex Lyapunov function. It does not violate the known $\EE[ \|h ( \frac{1}{n+1} \sum_{k=0}^n \prm_k)\|^2] = {\cal O}(1/n)$ rate in \citep{moulines2011non} as the latter uses  SA with a Polyak-Ruppert average estimator.
To our best knowledge, it remains an open problem to lower bound the convergence rate of SA for smooth but non-convex Lyapunov function. We mention here a recent work \citep[Remark 1]{fang2018spider} which shows $\EE[ \| h( \prm_n ) \|^2 ] = \Omega( 1/ \sqrt{n})$ under different conditions than those satisfied in this paper.\vspace{.2cm}


\begin{proof}
Our proof is achieved through constructing the Lyapunov and mean field function below. Consider a scalar parameter $\eta \in \rset$ and set $V(\eta)$ to be a $\mu$-strongly convex and $L$-smooth function, where $0 < \mu \leq L < \infty$. Also, the mean field is set as 
\beq
h(\eta) = V'(\eta) \eqs.
\eeq
Consider the following SA scheme \eqref{eq:sa} defined on the mean field $h$ as:
\beq \label{eq:sa_sp}
 \eta_{k+1}=\eta_k - \gamma_{k+1} \big(h(\eta_k) + e_{k+1} \big) \eqs,
\eeq
where $e_k$ is i.i.d.~and uniformly distributed on $[-\varepsilon,\varepsilon]$.

Clearly, the SA scheme \eqref{eq:sa_sp} satisfies A\ref{ass:SA2}-A\ref{ass:L} as we have set $V'(\eta) = h(\eta)$. The noise sequence is i.i.d.~satisfying A\ref{ass:iid1}-A\ref{ass:MC3}.
As $V$ is $\mu$-strongly convex, it can be shown
\beq
V(\eta_{k+1}) \geq V(\eta_k) - \gamma_{k+1} V'(\eta_k) \big( h(\eta_k) + e_{k+1} \big) + \gamma_{k+1}^2 \frac{\mu}{2} \big( h(\eta_k) + e_{k+1} \big)^2 \eqs.
\eeq
Now by construction, we have $\EE[ e_{k+1} V'(\eta_k) | {\cal F}_k ]=0$, $\EE[ \big( h(\eta_k) + e_{k+1} \big)^2 | {\cal F}_k ] \geq \frac{1}{3}\varepsilon^2$. Taking the total expectation on both sides gives
\beq
\EE [V(\eta_{k+1})] \geq \EE [V(\eta_k)] - \gamma_{k+1} h^2(\eta_k) + \gamma_{k+1}^2 \frac{\mu \varepsilon^2}{6} \eqs.
\eeq
Denote $C_{\sf lb} \eqdef \frac{\mu \varepsilon^2}{6}$. Using \eqref{eq:prob_stop}, we observe
\beq
\hspace{-.1cm}\EE[ | h(\eta_N)|^2 ] = \frac{1}{\sum_{k=0}^n \gamma_{k+1} } \sum_{k=0}^n \gamma_{k+1} \EE[ |h(\eta_k)|^2 ] \geq
\frac{\EE[ V(\eta_0) - V(\eta_{n+1}) ] +  C_{\sf lb} \sum_{k=0}^n \gamma_{k+1}^2}{\sum_{k=0}^n \gamma_{k+1}}.
\eeq
This completes the proof of the lower bound.
\end{proof}

\section{Analysis of the ro-EM method} \label{app:roem}

\subsection{Proof of Proposition~\ref{prop:hstar}}
\begin{Prop*}
Assume A\ref{ass:reg}. Then
\begin{itemize}
\item If $h( {\bm s}^\star ) = {\bm 0}$ for some $\bm{s}^\star \in \Sset$, then $\grd_{\param} \KL{\pi}{g_{\param^\star}}  + \grd_{ \param } \Pen( \param^\star )  = {\bm 0}$ with  $\param^\star = \mstep {\bm{s}^\star}$.
\item If $\grd_{\param} \KL{\pi}{g_{\param^\star}}  + \grd_{ \param } \Pen( \param^\star )  = {\bm 0}$ for some $\param^\star \in \Param$ then $\bm{s}^\star= \PE_\pi[S(Y,\param^\star)]$.
\end{itemize}
\end{Prop*}
\begin{proof}
We have
\beq
\grd_{\param} \KL{\pi}{g(\cdot; \param)}  = - \grd_{ \param} \EE_{\pi} \big[ \log g( Y; \param ) \big] =
- \EE_{ \pi } \big[ \grd_{ \param } \log g( Y; \param )  \big] \eqsp,
\eeq
where the last equality assumes that we can exchange integration with differentiation. Furthermore,
using the Fisher's identity \citep{douc2014nonlinear}, it holds for any $y \in {\sf Y}$ that
\beq
\grd_{ \param } \log g( y; \param ) = - \grd_{\param} \psi( \param ) + J_{ \phi }^{\param } ( \param ) \!~ \overline{\bm s}( y; \param ) = - \grd_{\param} \psi( \param ) + J_{ \phi }^{\param } ( \param ) \!~ \EE_{\param} \big[ S ({\bm X} ) | Y = y \big] \eqsp.
\eeq
Therefore, for any ${\bm s}$, it holds that
\begin{align}
& \grd_{\param} \KL{\pi}{g(\cdot; \overline{\param} ( {\bm s} ) )} + \grd_{\param} \Pen( \overline{\param} ( {\bm s} ) )
= \grd_{\param} \psi( \overline{\param} ( {\bm s} ) ) + \grd_{\param} \Pen( \overline{\param} ( {\bm s} ) ) - J_{\phi}^{\param} ( \overline{\param} ( {\bm s} ) ) \!~ \EE_{\pi} \big[ \overline{\bm s} ( Y; \overline{\param} ( {\bm s} ) ) \big]  \notag\\
& \overset{(a)}{=} J_{\phi}^{\param} ( \overline{\param} ( {\bm s} ) ) \Big( {\bm s} - \EE_{\pi} \big[ \overline{\bm s} ( Y; \overline{\param} ( {\bm s} ) ) \big] \Big)
\overset{(b)}{=} J_{\phi}^{\param} ( \overline{\param} ( {\bm s} ) ) \!~ h( {\bm s} ) \eqsp. \label{eq:KLgrad}
\end{align}
where we have used the assumption A\ref{ass:reg} in (a) and the definition of $h( {\bm s})$ in (b).
The conclusion  follows directly from the identity \eqref{eq:KLgrad} since $J_{\phi}^{\param} ( \overline{\param} ( {\bm s} ) )$ is full rank.
\end{proof}

\subsection{Proof of Proposition~\ref{prop:asat}}

\begin{Prop*}
Assume A\ref{ass:reg}. Then, for $\bm{s} \in \Sset$,
\beq
\grd_{ {\bm s} } V( {\bm s} ) = \jacob{ \phi }{ \param }{ \mstep{\bm s}} \Big( \hess{\ell}{\param}( {\bm s}; \param )  \Big)^{-1} \jacob{ \phi }{ \param }{ \mstep{\bm s}}^\top \!~ h ( {\bm s}) \eqsp.
\eeq
\end{Prop*}

\begin{proof}
Using chain rule and A\ref{ass:reg}, we obtain
\beq \label{eq:grd_v}
\begin{split}
\grd_{ \bm s} V( {\bm s} ) & = \jacob{ \overline{\param} }{ \bm s }{\bm s}^\top
\Big( \grd_{\param} \KL{\pi}{g(\cdot; \mstep{\bm s} )} + \grd_{\param} \Pen( \mstep{\bm s} ) \Big) \\
& =  \jacob{ \overline{\param} }{ \bm s }{\bm s}^\top \jacob{\phi}{\param}{ \mstep{\bm s} }^\top \!~ h( {\bm s} ) \eqsp,
\end{split}
\eeq
where the last equality uses the identity in \eqref{eq:KLgrad}. Consider the following vector map:
\beq
{\bm s} \to \grd_{\param} \psi ( \mstep{\bm s} ) + \grd_{ \param} \Pen( \mstep{\bm s} ) - \jacob{ \phi }{ \param }{\mstep{\bm s} }^\top \!~{\bm s} \eqsp.
\eeq
Taking the gradient of the above map \wrt ${\bm s}$ and note that the map is constant for all ${\bm s} \in \Sset$, we show that:
\beq
{\bm 0} = - \jacob{\phi}{\param}{\mstep{\bm s}} + \Big( \underbrace{ \grd_{\param}^2 \big( \psi( \param ) + \Pen( \param ) - \pscal{ \phi( \param ) }{ {\bm s} } \big)}_{= \hess{\ell}{\param} ( {\bm s}; \param )} \big|_{\param = \mstep{\bm s} } \Big) \jacob{ \overline{\param} }{\bm s}{\bm s} \eqsp.
\eeq
This implies $\jacob{ \overline{\param} }{\bm s}{\bm s} = \big( \hess{\ell}{\param} ( {\bm s}; \mstep{\bm s} ) \big)^{-1} \jacob{\phi}{\param}{\mstep{\bm s}}$. Substituting into \eqref{eq:grd_v} yields the conclusion.
\end{proof}

\subsection{Proof of Proposition~\ref{prop:variance}}
\begin{Prop*}
Under A\ref{ass:bdd},
it holds that $\EE[ \| \overline{s} ( Y_{n+1} ; \hat{\param}_n ) - \hat{\bm s}_n \|^2 | {\cal F}_n ] \leq 2M \overline{Y}^2$ for all $n \geq 0$.
\end{Prop*}
\begin{proof}
From \eqref{eq:hs}, we note that the error term is given by
\beq
{\bm e}_{n+1} = H_{ \hat{\bm s}_n } (Y_{n+1}) - h(\hat{\bm s}_n)
= \left(
\begin{array}{c}
\EE_{Y_{n+1} \sim \pi } [ \overline{\bm s}_n^{(1)} | {\cal F}_n ] - \overline{\bm s}_n^{(1)} \\
\EE_{Y_{n+1} \sim \pi }[ \overline{\bm s}_n^{(2)} | {\cal F}_n ] - \overline{\bm s}_n^{(2)} \\
\EE_{Y_{n+1} \sim \pi }[ \overline{s}_n^{(3)} | {\cal F}_n ] - \overline{s}_n^{(3)}
\end{array}
\right).
\eeq
Obviously, it holds that $\EE[ {\bm e}_{n+1} | {\cal F}_n ] = {\bm 0}$.
Furthermore, for all $m \in \{1,\dots,M-1\}$, the $m$th element of the first block in ${\bm e}_{n+1}$ has a bounded conditional variance
\beq
\EE \Big[ \big| \EE_{Y_{n+1}\sim \pi} [ \omega_m( Y_{n+1}; \hat{\param}_n ) ] - \omega_m( Y_{n+1}; \hat{\param}_n ) \big|^2 \Big] \leq 1 \eqsp.
\eeq
For the second block in ${\bm e}_{n+1}$, the conditional variance of its $m$th element is
\beq
\begin{split}
& \EE \Big[ \big| \EE_{Y_{n+1}\sim \pi} [ Y_{n+1} \omega_m( Y_{n+1}; \hat{\param}_n ) ] - Y_{n+1} \omega_m( Y_{n+1}; \hat{\param}_n ) \big|^2 \Big] \\
& = \EE \Big[ \big| Y_{n+1} \omega_m( Y_{n+1}; \hat{\param}_n ) \big|^2 \Big] - \big| \EE_{Y_{n+1}\sim \pi} [ Y_{n+1} \omega_m( Y_{n+1}; \hat{\param}_n ) ] \big|^2 \\
& \leq \EE \Big[ \big| Y_{n+1} \omega_m( Y_{n+1}; \hat{\param}_n ) \big|^2 \Big] \leq
\EE \big[ (Y_{n+1})^2 \big] \leq  \overline{Y}^2 .
\end{split}
\eeq
Lastly, we also have $\EE[ | \EE_{Y_{n+1} \sim \pi }[ \overline{s}_n^{(3)} | {\cal F}_n ] - \overline{s}_n^{(3)} |^2 ] \leq \overline{Y}^2$.
Therefore, we conclude that $\EE[ \| {\bm e}_{n+1} \|^2 | {\cal F}_n ] \leq M-1 + M \overline{Y}^2 < \infty$.
\end{proof}

\subsection{Proof of Proposition~\ref{prop:big}}
\begin{Prop*}
Under A\ref{ass:bdd} and the regularizer \eqref{eq:regu} set with $\epsilon > 0$, then for all $({\bm s},{\bm s}') \in \Sset^2$, there exists positive constants $\upsilon, \Upsilon, \Psi$ such that:
\beq
\pscal{ \grd V( {\bm s} ) }{ h( {\bm s} ) } \geq \upsilon \!~ \| h( {\bm s} ) \|^2,~~\| \grd V( {\bm s}) \| \leq \Upsilon \| h( {\bm s} ) \|,~~
\| \grd V( {\bm s} ) - \grd V({\bm s}') \| \leq \Psi \| {\bm s} - {\bm s}' \|.
\eeq
\end{Prop*}

\begin{proof}
We first check that A\ref{ass:reg} is satisfied under A\ref{ass:bdd}. In particular, one observes that when ${\bm s} \in \Sset= \Delta_{M-1} \times [-\overline{Y}, \overline{Y}]^M$, the {\sf M-step} update \eqref{eq:mstep_gmm} is the unique solution satisfying the stationary condition of the minimization problem \eqref{eq:oem} and   $\mstep{\bm s} \in \Cset$.

As A\ref{ass:reg} is satisfied, applying Proposition~\ref{prop:asat} shows that the gradient of the Lyapunov function is
\beq \label{eq:grad_s_pf}
\grd V( {\bm s} ) = \jacob{ \phi }{ \param }{ \mstep{\bm s}} \Big( \hess{\ell}{\param} ( {\bm s}; \param ) \big\}  \Big)^{-1} \jacob{ \phi }{ \param }{ \mstep{\bm s}}^\top \!~ h ( {\bm s}) \eqsp.
\eeq
Using \eqref{eq:gmm_exp}, we observe that
for any given $\param \in \Cset$, the Jacobian of $\phi$ and the Hessian of $\ell( {\bm s}, \param )$ are given by
\beq
\jacob{\phi}{\param}{\param} = \left(
\begin{array}{ccc}
\frac{1}{1 - \sum_{m=1}^{M-1} \omega_m} {\bf 1}{\bf 1}^\top + {\rm Diag}( \frac{\bf 1}{\bm{\omega}} ) & - {\rm Diag}( \bm{\mu} ) & \mu_M {\bf 1} \\
{\bm 0} & {\bm I} & {\bm 0} \\
{\bm 0} & {\bm 0} & 1
\end{array}
\right) \eqsp,
\eeq
\beq \notag
\hess{\ell}{\param} ( {\bm s}, \param ) =  \left(
\begin{array}{ccc}
\frac{1 + \epsilon - \sum_{m=1}^{M-1} s_m^{(1)} }{(1 - \sum_{m=1}^{M-1} \omega_m)^2} {\bf 1}{\bf 1}^\top + {\rm Diag}( \frac{ {\bm s}^{(1)} + \epsilon {\bf 1}}{\bm{\omega}^2} ) & {\bm 0} & {\bm 0} \\
{\bm 0} & {\rm Diag}( {\bm s}^{(1)} + \epsilon {\bf 1} ) & {\bm 0} \\
{\bm 0} & {\bm 0} & 1 + \epsilon - \sum_{m=1}^{M-1} s_m^{(1)}
\end{array}
\right) \eqsp,
\eeq
where we have denoted $\frac{ {\bm s}^{(1)} + \epsilon {\bf 1}}{\bm{\omega}^2} $ as the $(M-1)$-vector
$\big( \frac{ s_1^{(1)} + \epsilon }{ \omega_1^2 }, \ldots, \frac{ s_{M-1}^{(1)} + \epsilon }{ \omega_{M-1}^2 } \big)$.
Let us define ${\bm J}_{11}, {\bm H}_{11}$ as the top-left matrices in the above, evaluated at $\mstep{\bm s}$, as follows
\beq
{\bm J}_{11} \eqdef \frac{1}{1 - \frac{ {\bf 1}^\top ( {\bm s}^{(1)} + \epsilon {\bf 1} ) }{1+ \epsilon M} } {\bf 1}{\bf 1}^\top + {\rm Diag}( \frac{1 + \epsilon M}{{\bm s}^{(1)} + \epsilon {\bf 1}} )
\eeq
\beq
{\bm H}_{11} \eqdef
\frac{1 + \epsilon - \sum_{m=1}^{M-1} s_m^{(1)} }{(1 -  \frac{ {\bf 1}^\top ( {\bm s}^{(1)} + \epsilon {\bf 1} ) }{1+ \epsilon M} )^2} {\bf 1}{\bf 1}^\top + {\rm Diag}( \frac{ (1+\epsilon M)^2 }{ {\bm s}^{(1)} + \epsilon {\bf 1} } ).
\eeq
When $\epsilon > 0$, the above matrices, ${\bm J}_{11}$ and ${\bm H}_{11}$, are full rank and bounded if ${\bm s} \in \Sset$.

The matrix product $\jacob{\phi}{\param}{\mstep{\bm s} } \big( \hess{\ell}{\param} ( {\bm s}, \mstep{\bm s} ) \big)^{-1} \jacob{\phi}{\param}{\mstep{\bm s} }^\top$ can hence be expressed as an outer product
\beq
\jacob{\phi}{\param}{ \mstep{\bm s} } \big( \hess{\ell}{\param} ( {\bm s}, \mstep{\bm s} )  \big)^{-1} \jacob{\phi}{\param}{ \mstep{\bm s} }^\top = \bm{\mathcal{J}} ( {\bm s} ) \bm{\mathcal{J}} ( {\bm s} )^\top \eqsp,
\eeq
with
\beq \label{eq:Jdef}
\begin{split}
\bm{\mathcal{J}} ( {\bm s}  ) & \eqdef \jacob{\phi}{\param}{ \mstep{\bm s} } \left(
\begin{array}{ccc}
{\bm H}_{11}^{-\frac{1}{2}} & {\bm 0} & {\bm 0} \\
{\bm 0} & {\rm Diag}( \frac{ \bm 1}{ \sqrt{ {\bm s}^{(1)} + \epsilon {\bf 1} }} ) & {\bm 0} \\
{\bm 0} & {\bm 0} & \frac{1}{\sqrt{1 + \epsilon - \sum_{m=1}^{M-1} s_m^{(1)}}}
\end{array}
\right) \\
& = \left(
\begin{array}{ccc}
{\bm J}_{11} {\bm H}_{11}^{-\frac{1}{2}} &  -{\rm Diag} \big( \frac{ {\bm s}^{(2)} }{ ( {\bm s}^{(1)} + \epsilon {\bf 1} )^{\frac{3}{2}} }  \big)  & \frac{ s^{(3)} - {\bm 1}^\top {\bm s}^{(2)} }{ (1+\epsilon-\sum_{m=1}^{M-1} s_m^{(1)})^{\frac{3}{2}} } {\bf 1} \\
{\bm 0} & {\rm Diag}( \frac{ \bm 1}{ \sqrt{ {\bm s}^{(1)} + \epsilon {\bf 1} }} ) & {\bm 0} \\
{\bm 0} & {\bm 0} & \frac{1}{\sqrt{1 + \epsilon - \sum_{m=1}^{M-1} s_m^{(1)}}}
\end{array}
\right) \eqsp.
\end{split}
\eeq
Under A\ref{ass:bdd} and using the above structured form, it can be verified that $\bm{\mathcal{J}} ( {\bm s}  )$ is a bounded and full rank matrix. As such, for all ${\bm s} \in \Sset$, there exists $\upsilon > 0$ such that
\beq
\pscal{ \grd V( {\bm s}  ) }{ h( {\bm s} ) } = \pscal{ \bm{\mathcal{J}} ( {\bm s}  ) \bm{\mathcal{J}} ( {\bm s}  )^\top h ({\bm s}) }{ h( {\bm s} ) } \geq \upsilon \!~ \| h( {\bm s} ) \|^2 \eqsp.
\eeq
The second part in \eqref{eq:big_res} can be verified by observing that
$\jacob{ \phi }{ \param }{ \mstep{\bm s}} \Big( \hess{\ell}{\param} ( {\bm s}; \param ) \big\}  \Big)^{-1} \jacob{ \phi }{ \param }{ \mstep{\bm s}}^\top$ is bounded due to A\ref{ass:bdd}.

For the third part in \eqref{eq:big_res}, again from \eqref{eq:grad_s_pf} we obtain:
\beq
\grd V({\bm s}) = \bm{\mathcal{J}} ( {\bm s} ) \bm{\mathcal{J}} ( {\bm s} )^\top h( {\bm s} ) \eqsp.
\eeq
From \eqref{eq:Jdef}, it can be seen that $\bm{\mathcal{J}} ( {\bm s} ) \bm{\mathcal{J}} ( {\bm s} )^\top$ is Lipschitz continuous in ${\bm s}$ and bounded, \ie there exists constants $L_J, C_J < \infty$ such that
\beq
\| \bm{\mathcal{J}} ( {\bm s} ) \bm{\mathcal{J}} ( {\bm s} )^\top - \bm{\mathcal{J}} ( {\bm s}' ) \bm{\mathcal{J}} ( {\bm s}' )^\top \| \leq L_J \| {\bm s} - {\bm s}' \|,~~\| \bm{\mathcal{J}} ( {\bm s} ) \bm{\mathcal{J}} ( {\bm s} )^\top \| \leq C_J,~\forall~{\bm s}, {\bm s}' \in \Sset \eqsp.
\eeq
For example, the above can be checked by observing that the Hessian (\wrt ${\bm s}$) of each entry in $\bm{\mathcal{J}} ( {\bm s} ) \bm{\mathcal{J}} ( {\bm s} )^\top$ is bounded for ${\bm s} \in \Sset$. On the other hand, the mean field $h( {\bm s} )$ satisfies,
\beq \label{eq:h_lips}
\begin{split}
\| h( {\bm s} ) - h( {\bm s}' ) \| & = \| {\bm s} - {\bm s}' + \EE_{Y \sim \pi} \big[ \overline{\bm s} (Y; \mstep{ {\bm s}' }) - \overline{\bm s} (Y; \mstep{ {\bm s} })  \big] \| \\
& \overset{(a)}{\leq} \| {\bm s} - {\bm s}'  \| + \EE_{Y \sim \pi} \big[ \| \overline{\bm s} (Y; \mstep{ {\bm s}' }) - \overline{\bm s} (Y; \mstep{ {\bm s} }) \| \big] \eqsp,
\end{split}
\eeq
where (a) uses the triangular inequality and the Jensen's inequality. Moreover, we observe
\beq \label{eq:individual_omega}
\overline{\bm s} (Y; \mstep{ {\bm s}' }) - \overline{\bm s} (Y; \mstep{ {\bm s} })
= \left(
\begin{array}{c}
\widetilde{\bm{\omega}} ( Y; \mstep{{\bm s}'} ) - \widetilde{\bm{\omega}} ( Y; \mstep{\bm s} ) \\
Y \big( \widetilde{\bm{\omega}} ( Y; \mstep{{\bm s}'} ) - \widetilde{\bm{\omega}} ( Y; \mstep{\bm s} ) \big) \\
0
\end{array}
\right) \eqsp,
\eeq
where $\widetilde{\bm{\omega}}( Y; \mstep{\bm s} )$ is a collection of the $M-1$ terms $\widetilde{{\omega}}_m( Y; \mstep{\bm s} )$, $m=1,\dots,M-1$ [cf.~\eqref{eq:cexp}]. Observe that
\beq
\widetilde{{\omega}}_m( Y; \mstep{\bm s} ) = \frac{ \frac{ s_m^{(1)} + \epsilon}{1+\epsilon M}  \!~ {\rm exp}(-\frac{1}{2}( Y - \frac{ s_m^{(2)} }{ s_m^{(1)} + \epsilon } )^2) }{  \sum_{j=1}^{M} \frac{ s_j^{(1)} + \epsilon}{1+\epsilon M} \!~ \exp(-\frac{1}{2}( Y - \frac{ s_j^{(2)} }{ s_j^{(1)} + \epsilon } )^2) } \eqsp.
\eeq
Under A\ref{ass:bdd} and the condition that ${\bm s} \in \Sset$, \ie a compact set, there exists $L_{\omega}  <\infty$ such that
\beq
|\widetilde{{\omega}}_m( Y; \mstep{\bm s} ) - \widetilde{{\omega}}_m( Y; \mstep{ {\bm s}' } ) | ^2 \leq L_{\omega}^2 \| {\bm s} - {\bm s}' \|^2 \eqsp,
\eeq
for all $m=1,\dots,M-1$. Consequently, again using A\ref{ass:bdd}, we have
\beq
\| \overline{\bm s} (Y; \mstep{ {\bm s}' }) - \overline{\bm s} (Y; \mstep{ {\bm s} }) \|
\leq (M-1)( 1 +  \overline{Y} ) L_\omega \| {\bm s} - {\bm s}' \| \eqsp,
\eeq
and we have $\| h( {\bm s} ) - h( {\bm s}' ) \|  \leq L_h \| {\bm s} - {\bm s}' \|$
for some $L_h < \infty$. It can also be shown easily that $\| h( {\bm s} ) \| \leq C_h$
for all ${\bm s} \in \Sset$.
Finally, we observe the following chain:\vspace{-.2cm}
\beq
\begin{split}
& \| \grd V( {\bm s} )  - \grd V ({\bm s}') \| = \| \bm{\mathcal{J}} ( {\bm s} ) \bm{\mathcal{J}} ( {\bm s} )^\top h( {\bm s} ) - \bm{\mathcal{J}} ( {\bm s}' ) \bm{\mathcal{J}} ( {\bm s}' )^\top h( {\bm s}' ) \| \\
& =  \| \bm{\mathcal{J}} ( {\bm s} ) \bm{\mathcal{J}} ( {\bm s} )^\top ( h( {\bm s} ) - h( {\bm s}' ) ) + \big( \bm{\mathcal{J}} ( {\bm s} ) \bm{\mathcal{J}} ( {\bm s} )^\top  - \bm{\mathcal{J}} ( {\bm s}' ) \bm{\mathcal{J}} ( {\bm s}' )^\top \big) h( {\bm s}' ) \| \\
& \leq \big( L_h C_J + L_J C_h \big) \| {\bm s} - {\bm s}' \| ,
\end{split}
\eeq
which concludes our proof.
\end{proof}

\section{Analysis on the Policy Gradient Algorithm} \label{app:pg}

This section proves a few key lemmas that are modified from \citep{tadic2017asymptotic} which leads to the convergence of the policy gradient algorithm analyzed in Section~\ref{sec:policy_grad}.

Let $\tilde{\bm Q}_{\prm} \eqdef {\bm Q}_{\prm} - {\bf 1} {\bm{\upsilon}}_{\prm}^\top$ and denote $\tilde{Q}_{\prm}^t ( (s,a) ; (s',a') )$ to be the $( (s,a), (s',a') )$th element of the $t$th power of $\tilde{\bm Q}_{\prm}^t$. Under A\ref{ass:markov}, we observe that $\| \tilde{\bm Q}_{\prm}^t \| \leq \rho^t K_R$ for any $t \geq 0$.
For $i=1,...,d$, we also define the $(s,a)$th element of the $|{\cal S}| |{\cal A}|$-dimensional gradient vector $\grd_i \bm{\Pi}_{\prm}$, and reward vector ${\bm r}$, respectively as:
\beq
\grd_i \bm{\Pi}_{\prm} ( s, a ) \eqdef \frac{ \partial \log \Pi ( a ; s , \prm ) }{ \partial \eta_i },~~
r(s,a) \eqdef {\cal R}(s,a).
\eeq
Using the above notations, the mean field in \eqref{eq:SA_policy} can be evaluated as
\beq
h ( \prm ) = \sum_{t=0}^\infty \sum_{ (s,a) ,( s', a') \in {\cal S} \times {\cal A} } \lambda^t {\cal R}(s',a') \tilde{Q}_{\prm}^t ( (s,a) ; (s',a') ) \grd \log \Pi( a ; s, \prm ) \upsilon_{\prm} (  s,a ).
\eeq
In particular, its $i$th element can be expressed as
\beq
h_i ( \prm ) = \sum_{t=0}^\infty \lambda^t \bm{\upsilon}_{\prm}^\top {\rm Diag}( \grd_i \bm{\Pi}_{\prm} ) \tilde{\bm Q}_{\prm}^t  {\bm r} \eqsp.
\eeq
We also define the difference between $h(\prm)$ and $\grd J(\prm)$ as
\beq
\Delta( \prm ) \eqdef h(\prm) - \grd J(\prm).
\eeq
\subsection{Useful Lemmas}
\begin{Lemma} \label{lem:lemma81}
Let A\ref{ass:bdd_feat}, A\ref{ass:markov} hold. For any $(\prm, \prm') \in \Prm^2$ and $t \geq 0$, one has
\beq
\| {\bm Q}_{\prm}^t - {\bm Q}_{\prm'}^t \| \leq C_1 \| \prm - \prm' \|,~~\| \tilde{\bm Q}_{\prm}^t - \tilde{\bm Q}_{\prm'}^t \| \leq C_1 \big( t \rho^t \big) \| \prm - \prm' \| \eqsp,
\eeq
where we have set $C_1 \eqdef \rho K_R^2 \big( 2 \overline{b} +L_Q \big) + L_Q$ in the above.
\end{Lemma}

\begin{proof}
For part 1), we observe that each entry of ${\bm Q}_{\prm}$ is given by [cf.~\eqref{eq:rprm}]:
\[
Q_{\prm} ( (s,a) ; (s',a') ) \eqdef \Pi( a' ; s', \prm ) {P}_{s,s'}^a  \eqsp,
\]
which is Lipschitz continuous \wrt $\prm$ since
\beq \notag
\begin{split}
& \grd \Pi( a | s, \prm ) = \\
& - \big( \sum_{a' \in {\cal A}} \exp \big( \pscal{\prm}{ {\bm x}(s,a') - {\bm x}(s,a) } \big) \big)^{-2} \sum_{a' \in {\cal A}} \exp \big( \pscal{\prm}{ {\bm x}(s,a') - {\bm x}(s,a) } \big) ( {\bm x}(s,a') - {\bm x}(s,a) )
\end{split}
\eeq
is bounded by $\max_{s,a,a'} \| {\bm x}(s,a') - {\bm x}(s,a) \| \leq 2 \overline{b}$ [cf.~A\ref{ass:bdd_feat}]. This implies
\beq
| Q_{\prm} ( (s,a) ; (s',a') ) - Q_{\prm'} ( (s,a) ; (s',a') ) | \leq 2 \overline{b} |{P}_{s,s'}^a| \!~ \| \prm - \prm' \| \eqsp.
\eeq
Since $|{P}_{s,s'}^a| \leq 1$ for any $s,s',a$, we have $\| {\bm Q}_{\prm} - {\bm Q}_{\prm'} \| \leq 2 \overline{b} \| \prm - \prm' \|$.

For any $\prm \in \Prm$
and any $t \geq 0$, we have:
\beq
\begin{split}
\tilde{\bm Q}_{\prm}^{t+1} - \tilde{\bm Q}_{\prm'}^{t+1} & = \sum_{\tau=0}^t \tilde{\bm Q}_{\prm}^\tau \big( \tilde{\bm Q}_{\prm} - \tilde{\bm Q}_{\prm'} \big) \tilde{\bm Q}_{\prm'}^{t-\tau} \\
& = \sum_{\tau=0}^t \tilde{\bm Q}_{\prm}^\tau \big( {\bm Q}_{\prm} - {\bm Q}_{\prm'} - {\bf 1} (\bm{\upsilon}_{\prm} - \bm{\upsilon}_{\prm'} )^\top \big) \tilde{\bm Q}_{\prm'}^{t-\tau} \eqsp.
\end{split}
\eeq
As such,
\beq
\begin{split}
\| \tilde{\bm Q}_{\prm}^{t+1} - \tilde{\bm Q}_{\prm'}^{t+1} \| & \leq \sum_{\tau=0}^t \| \tilde{\bm Q}_{\prm}^\tau \| \big\| {\bm Q}_{\prm} - {\bm Q}_{\prm'} - {\bf 1} (\bm{\upsilon}_{\prm} - \bm{\upsilon}_{\prm'} )^\top \big\| \| \tilde{\bm Q}_{\prm'}^{t-\tau} \| \\
& \leq K_R^2 \sum_{\tau=0}^t \rho^\tau \rho^{t-\tau} \big( \| {\bm Q}_{\prm} - {\bm Q}_{\prm'} \| + \| \bm{\upsilon}_{\prm} - \bm{\upsilon}_{\prm'} \| \big) \\
& \leq K_R^2 \big( 2 \overline{b} + L_Q \big) \big( t \!~ \rho^t \big)  \| \prm - \prm' \| \eqsp.
\end{split}
\eeq
Consequently,
\beq
\begin{split}
\| {\bm Q}_{\prm}^{t+1} - {\bm Q}_{\prm'}^{t+1} \| & \leq \| \tilde{\bm Q}_{\prm}^{t+1} - \tilde{\bm Q}_{\prm'}^{t+1} \|  + \| \bm{\upsilon}_{\prm} - \bm{\upsilon}_{\prm'} \| \\
& \leq \big(  K_R^2 \big( t \!~ \rho^t \big) \big( 2 \overline{b} + L_Q  \big) + L_Q \big) \| \prm - \prm' \| \eqsp.
\end{split}
\eeq
Setting $C_1 = \rho K_R^2 \big( 2 \overline{b} +L_Q \big) + L_Q$ completes the proof.
\end{proof}

\begin{Lemma} \label{lem:lemma82}
Let A\ref{ass:bdd_feat}, A\ref{ass:markov} hold. The following statements are true:
\begin{enumerate}
\item The average reward $J(\prm)$ is differentiable and for any $(\prm, \prm') \in \Prm^2$, one has
\beq \label{eq:L_J}
\| \grd J( \prm ) - \grd J(\prm' ) \| \leq \Reward_{\max} \!~ |{\cal S}| |{\cal A}| \!~ L_{\upsilon} \| \prm - \prm' \| \eqsp.
\eeq
\item For any $\prm \in \Prm$, one has
\beq
\| \Delta ( \prm ) \| \leq 2 \overline{b} \!~ \Reward_{\max} K_R \frac{ 1 - \lambda }{ (1 - \rho)^2} \eqsp.
\eeq
\end{enumerate}
\end{Lemma}

\begin{proof}
For part 1), we observe that
\beq
J( \prm ) = \EE_{(S,A) \sim \bm{\upsilon}_{\prm} } \big[ {\cal R}( S, A)  \big] = \sum_{ (s,a) \in {\cal S} \times {\cal A} } \upsilon_{\prm} ( s,a ) {\cal R}(s,a) \eqsp.
\eeq
It follows from the Lipschitz continuity of $\jacob{\bm{\upsilon}_{\prm}}{\prm}{\prm}$ [cf.~A\ref{ass:markov}] that
\beq
\begin{split}
\| \grd J( \prm ) - \grd J( \prm' ) \| & \leq \sum_{ (s,a) \in {\cal S} \times {\cal A} } | {\cal R} ( s,a ) | \| \grd \upsilon_{\prm} ( s,a ) - \grd \upsilon_{\prm'} ( s,a ) \| \\
& \leq \Reward_{\max} \!~ |{\cal S}| |{\cal A}| \!~ L_{\upsilon} \!~ \| \prm - \prm' \| \eqsp.
\end{split}
\eeq
The above verifies \eqref{eq:L_J}.

For part 2), we define
\beq
J_T( \prm, (s,a) ) \eqdef \sum_{ (s',a') \in {\cal S} \times {\cal A} } {\cal R}( s',a') Q_{\prm}^T ( (s,a)  ; ( s',a' ) ) \eqsp,
\eeq
\beq
g( \prm ) \eqdef \sum_{t=0}^\infty \sum_{ (s,a), (s',a') \in {\cal S} \times {\cal A}} {\cal R}( s,a) \tilde{Q}_{\prm}^t ( ( s,a) ; (s',a') ) \grd \log \Pi( a ; s, \prm ) \upsilon_{\prm} (  s,a )\eqsp.
\eeq
As shown in \citep[Lemma 8.2]{tadic2017asymptotic}, we have $\lim_{T \rightarrow \infty} \grd_{\prm} J_T( \prm, (s,a) ) = g( \prm )$ for all $\prm \in \Prm$ and $(s,a) \in {\cal S} \times {\cal A}$. As such
\beq
\begin{split}
& \Delta( \prm ) = h( \prm ) - g( \prm ) \\
& = \sum_{t=0}^\infty \sum_{ (s,a), (s',a') \in {\cal S} \times {\cal A}} ( \lambda^t - 1 ) {\cal R}( s,a) \tilde{Q}_{\prm}^t ( ( s,a) ; (s',a') ) \grd \log \Pi( a ; s, \prm ) \upsilon_{\prm} (  s,a ) \eqsp.
\end{split}
\eeq
and in particular, the $i$th element is given by
\beq
\Delta_i( \prm) = \sum_{t=0}^\infty \sum_{ (s,a) ,( s', a') \in {\cal S} \times {\cal A} } \big( \lambda^t - 1\big) \bm{\upsilon}_{\prm}^\top {\rm Diag}( \grd_i \bm{\Pi}_{\prm} )  \tilde{\bm Q}_{\prm}^t {\bm r} \eqsp,
\eeq
which can be bounded as
\beq
\begin{split}
| \Delta_i ( \prm ) | & \leq \sum_{t=0}^\infty ( 1 - \lambda^t  ) \| \bm{\upsilon}_{\prm}\|   \| \grd_i \bm{\Pi}_{\prm} \|_\infty  \| \tilde{\bm Q}_{\prm}^t \| \| {\bm r} \| \\
& \overset{(a)}{\leq} 2 \overline{b} \!~ \Reward_{\max} K_R \sum_{t=0}^\infty  (1 - \lambda^t) \rho^t \leq 2 \overline{b} \!~ \Reward_{\max} K_R \frac{ 1 - \lambda }{ (1 - \rho)^2} \eqsp,
\end{split}
\eeq
where (a) uses A\ref{ass:markov}, A\ref{ass:bdd_feat}, and Proposition~\ref{prop:bdd_grad}. The above implies that $\| \Delta ( \prm) \| \leq
2 \overline{b} \!~ \Reward_{\max} K_R \frac{ 1 - \lambda }{ (1 - \rho)^2}$.
\end{proof}

\begin{Lemma} \label{lem:lemma83}
Let A\ref{ass:bdd_feat}, A\ref{ass:markov} hold. Denote the joint state $x$ as $x = (s,a,g) \in {\cal S} \times {\cal A} \times \rset^d$. There exists $\delta \in [0,1)$, $C_2 \in [1,\infty)$ such that for any $t \geq 0$, 
\beq \label{eq:lemma83}
\begin{array}{l}
\| \PPX{t}{} \HX{\prm}{x} - h( \prm ) \| \leq C_2 \!~ t \!~ \delta^t ( 1 + \| g \| ) \vspace{.2cm} \eqsp,\\
\Big\| \big( \PPX{t}{} \HX{\prm}{x} - h( \prm ) \big) - \big( P_{\prm'}^t \HX{\prm'}{x} - h( \prm' ) \big) \Big\| \leq C_2 \!~ t \!~ \delta^t  \!~ \| \prm - \prm' \| ( 1 + \| g \| ) \eqsp.
\end{array}
\eeq
Moreover, we have $\delta = \max\{ \rho, \lambda \}$. 
\end{Lemma}

\begin{proof}
Denote the joint state as $x = (s,a,g)$, we observe that
\beq \notag
\begin{split}
& \PPX{t}{} \HX{\prm}{x} = \EE_{\Pi_{\prm}} \big[ {\cal R}( S_t, A_t ) G_t \!~ | \!~ (S_0,A_0) = (s,a), G_0 = g \big] \\
& = \EE_{\Pi_{\prm}} \left[ {\cal R}( S_t, A_t ) \Big( \lambda^t g + \sum_{i=1}^{t-1} \lambda^i  \grd \log \Pi( A_i ; S_i , \prm ) \Big) \!~ \big| \!~ (S_0,A_0) = (s,a) \right] \\
& = \sum_{i=0}^{t-1} \sum_{ (s',a'), (s'', a'') \in {\cal S} \times {\cal A}} \hspace{-.6cm} \lambda^i {\cal R}( s'', a'' ) Q_{\prm}^i ( (s',a') ; (s'',a'') ) \grd \log \Pi( a' ; s' , \prm ) Q_{\prm}^{t-i} ( (s,a) ; (s',a') ) \\
& \hspace{.4cm}+ \lambda^t g \sum_{(s',a') \in {\cal S} \times {\cal A}} {\cal R}(s',a') Q_{\prm}^t ( (s,a) ; (s',a') ) \eqsp.
\end{split}
\eeq
The $j$th element of the above is thus given by
\beq
\big[ \PPX{t}{} \HX{\prm}{x} \big]_j = \sum_{i=0}^{t-i} \lambda^i  {\bm e}_{(s,a)}^\top {\bm Q}_{\prm}^{t-i}  {\rm Diag}( \grd_j \bm{\Pi}_{\prm} ) {\bm Q}_{\prm}^i {\bm r}
+ \lambda^t g_j {\bf 1}^\top {\bm Q}_{\prm}^t {\bm r} \eqsp,
\eeq
where $g_j$ is the $j$th element of $g$ and ${\bm e}_{(s,a)}$ is the $(s,a)$th coordinate vector. Moreover, we recall that
\beq
h_j(\prm) = \sum_{t=0}^\infty \lambda^t \bm{\upsilon}_{\prm}^\top {\rm Diag}( \grd_j \bm{\Pi}_{\prm} ) \tilde{\bm Q}_{\prm}^t {\bm r} \eqsp.
\eeq
Note that
\beq
\begin{split}
& \bm{\upsilon}_{\prm}^\top {\rm Diag}( \grd_j \bm{\Pi}_{\prm} ) {\bm 1}  = \sum_{ (s,a) \in {\cal S} \times {\cal A}} \upsilon_{\prm} ( s, a ) \grd_j \log \Pi ( a ; s , \prm ) \\
& = \sum_{s \in {\cal S}} \Big( \sum_{a \in {\cal A}} \underbrace{ \Pi( a ; s, \prm )  \grd_j \log \Pi ( a ; s , \prm )}_{= \grd_j \Pi( a ; s, \prm )} \Big) \overline{\Pi}_{\prm} (s) = 0 \eqsp.
\end{split}
\eeq
where we recalled that $\overline{\Pi}_{\prm} (s)$ is the stationary distribution for the MDP on the state.
Using the decomposition $\tilde{\bm Q}_{\prm}^t = {\bm Q}_{\prm}^t - {\bm 1} \bm{\upsilon}_{\prm}^\top$, we observe
\beq \notag
\begin{split}
h_j(\prm) & = \sum_{i=0}^{t-1} \lambda^i \Big\{  \bm{\upsilon}_{\prm}^\top {\rm Diag}( \grd_j \bm{\Pi}_{\prm} ) {\bm Q}_{\prm}^i  {\bm r} - \underbrace{\bm{\upsilon}_{\prm}^\top {\rm Diag}( \grd_j \bm{\Pi}_{\prm} ) {\bf 1}}_{=0} \bm{\upsilon}_{\prm}^\top {\bm r} \Big\} + \sum_{i=t}^\infty \lambda^i \bm{\upsilon}_{\prm}^\top {\rm Diag}( \grd_j \bm{\Pi}_{\prm} ) \tilde{\bm Q}_{\prm}^i  {\bm r}\\
& = \sum_{i=0}^{t-1} \lambda^i \bm{\upsilon}_{\prm}^\top {\rm Diag}( \grd_j \bm{\Pi}_{\prm} ) {\bm Q}_{\prm}^i  {\bm r} + \sum_{i=t}^\infty \lambda^i \bm{\upsilon}_{\prm}^\top {\rm Diag}( \grd_j \bm{\Pi}_{\prm} ) \tilde{\bm Q}_{\prm}^i  {\bm r} \eqsp.
\end{split}
\eeq
Therefore,
\beq
\begin{split}
& \big[ \PPX{t}{} \HX{\prm}{x} \big]_j - h_j(\prm) \\
& = \sum_{i=0}^{t-1} \lambda^i \Big\{ {\bm e}_{(s,a)}^\top ( \tilde{\bm Q}_{\prm}^{t-i} + {\bf 1} \bm{\upsilon}_{\prm} )  {\rm Diag}( \grd_j \bm{\Pi}_{\prm} ) {\bm Q}_{\prm}^i {\bm r} - \bm{\upsilon}_{\prm}^\top {\rm Diag}( \grd_j \bm{\Pi}_{\prm} ) {\bm Q}_{\prm}^i  {\bm r} \Big\} \\
& \hspace{.4cm} + \lambda^t g_j {\bf 1}^\top {\bm Q}_{\prm}^t {\bm r} - \sum_{i=t}^\infty \lambda^i \bm{\upsilon}_{\prm}^\top {\rm Diag}( \grd_j \bm{\Pi}_{\prm} ) \tilde{\bm Q}_{\prm}^i  {\bm r} \\
& = \sum_{i=0}^{t-1} \lambda^i {\bm e}_{(s,a)}^\top  \tilde{\bm Q}_{\prm}^{t-i} {\rm Diag}( \grd_j \bm{\Pi}_{\prm} ) {\bm Q}_{\prm}^i {\bm r} + \lambda^t g_j {\bf 1}^\top {\bm Q}_{\prm}^t {\bm r} - \sum_{i=t}^\infty \lambda^i \bm{\upsilon}_{\prm}^\top {\rm Diag}( \grd_j \bm{\Pi}_{\prm} ) \tilde{\bm Q}_{\prm}^i  {\bm r} \eqsp.
\end{split}
\eeq
Consequently, we obtain the upper bound as
\beq
\begin{split}
\big|  \big[ \PPX{t}{} \HX{\prm}{x} \big]_j - h_j(\prm)  \big| & \leq
\sum_{i=0}^{t-1} \lambda^i \| \tilde{\bm Q}_{\prm}^{t-i} \| \| \grd_j \bm{\Pi}_{\prm} \|_\infty \| {\bm Q}_{\prm}^i  {\bm r} \| + \lambda^t |g_j| \| {\bm Q}_{\prm}^t {\bm r} \| \\
& \hspace{.4cm}+ \sum_{i=t}^\infty \lambda^i \| \bm{\upsilon}_{\prm} \| \| \grd_j \bm{\Pi}_{\prm}  \|_\infty \| \tilde{\bm Q}_{\prm}^i {\bm r} \| \eqsp.
\end{split}
\eeq
Using A\ref{ass:bdd_feat}, A\ref{ass:markov} and notice that $ \| \grd_j \bm{\Pi}_{\prm} \|_\infty \leq 2\overline{b}$, $\| {\bm Q}_{\prm}^i {\bm r} \| \leq \overline{R}$, $\| \tilde{\bm Q}_{\prm}^i {\bm r} \| \leq \overline{R} K_R \sqrt{ |{\cal S}| | {\cal A} |} \rho^i$, we obtain
\beq
\begin{split}
\big|  \big[ \PPX{t}{} \HX{\prm}{x} \big]_j - h_j(\prm)  \big| & \leq 2 \overline{b} \!~ \overline{R} \!~ K_R \sum_{i=0}^{t-1} \lambda^i \rho^{t-i} + \lambda^t |g_j| \overline{R} + 2\overline{b}   \!~ \overline{R} K_R  \sqrt{|{\cal S}| | {\cal A} |} \sum_{i=t}^\infty \lambda^i \rho^i \eqsp.
\end{split}
\eeq
Observe that each of the above term decays geometrically with $t$ at the rate $\max\{ \rho, \lambda \}$, as such there exists $C_2' \in [1,\infty)$, $\delta = \max\{ \rho, \lambda \}  \in [0,1)$ such that\footnote{Note that an exact characterization for $C_2'$ is also possible.}
\beq
\big|  \big[ \PPX{t}{} \HX{\prm}{x} \big]_j - h_j(\prm)  \big| \leq C_2' \big(t \delta^t \big) \big( 1 + \| g \| \big) \eqsp,
\eeq
which naturally implies the first equation in \eqref{eq:lemma83}.

For the second equation in \eqref{eq:lemma83},
\beq
\begin{split}
& \big[ \PPX{t}{} \HX{\prm}{x} \big]_j - h_j(\prm) - \Big\{ \big[ P_{\prm'}^t \HX{\prm'}{x} \big]_j - h_j(\prm') \Big\} \\
& = \sum_{i=0}^{t-1} \lambda^i {\bm e}_{(s,a)}^\top \big\{ \tilde{\bm Q}_{\prm}^{t-i} {\rm Diag}( \grd_j \bm{\Pi}_{\prm} ) {\bm Q}_{\prm}^i - \tilde{\bm Q}_{\prm'}^{t-i} {\rm Diag}( \grd_j \bm{\Pi}_{\prm'} ) {\bm Q}_{\prm'}^i \big\} {\bm r} \\
& \hspace{.4cm} + \lambda^t g_j {\bf 1}^\top \big( {\bm Q}_{\prm}^t - {\bm Q}_{\prm'}^t \big) {\bm r} + \sum_{i=t}^\infty \lambda^i \big\{ \bm{\upsilon}_{\prm'}^\top {\rm Diag}( \grd_j \bm{\Pi}_{\prm'} ) \tilde{\bm Q}_{\prm'}^i - \bm{\upsilon}_{\prm}^\top {\rm Diag}( \grd_j \bm{\Pi}_{\prm} ) \tilde{\bm Q}_{\prm}^i \big\} {\bm r} \eqsp.
\end{split}
\eeq
This leads to the upper bound:
\beq \label{eq:upper_bd_lemma83}
\begin{split}
& \Big| \big[ \PPX{t}{} \HX{\prm}{x} \big]_j - h_j(\prm) - \Big\{ \big[ P_{\prm'}^t \HX{\prm'}{x} \big]_j - h_j(\prm') \Big\} \Big| \\
& \leq \sqrt{ |{\cal S}| |{\cal A}| } \overline{R} \!~\sum_{i=0}^{t-i} \lambda^i \big\| \tilde{\bm Q}_{\prm}^{t-i} {\rm Diag}( \grd_j \bm{\Pi}_{\prm} ) {\bm Q}_{\prm}^i - \tilde{\bm Q}_{\prm'}^{t-i} {\rm Diag}( \grd_j \bm{\Pi}_{\prm'} ) {\bm Q}_{\prm'}^i \big\| \\
& \hspace{.4cm} + \lambda^t |{\cal S}||{\cal A}| \| {\bm Q}_{\prm}^t - {\bm Q}_{\prm'}^t  \|
+ \sqrt{ |{\cal S}| |{\cal A}| } \overline{R} \!~ \sum_{i=t}^\infty \lambda^i \| \bm{\upsilon}_{\prm'}^\top {\rm Diag}( \grd_j \bm{\Pi}_{\prm'} ) \tilde{\bm Q}_{\prm'}^i - \bm{\upsilon}_{\prm}^\top {\rm Diag}( \grd_j \bm{\Pi}_{\prm} ) \tilde{\bm Q}_{\prm}^i \| \eqsp.
\end{split}
\eeq
Using the boundedness and Lipschitz continuity of $\grd_j \bm{\Pi}_{\prm}$, $\bm{\upsilon}_{\prm}$, ${\bm Q}_{\prm}^t$, $\tilde{\bm Q}_{\prm}^t$ [cf.~Lemma~\ref{lem:lemma81}], let $C_{2,1}, C_{2,2} \in [1,\infty)$, the norms in the above can be bounded as
\beq
\begin{split}
& \big\| \tilde{\bm Q}_{\prm}^{t-i} {\rm Diag}( \grd_j \bm{\Pi}_{\prm} ) {\bm Q}_{\prm}^i - \tilde{\bm Q}_{\prm'}^{t-i} {\rm Diag}( \grd_j \bm{\Pi}_{\prm'} ) {\bm Q}_{\prm'}^i \big\| \leq C_{2,1} \big( (t-i) \rho^{t-i}  \big)\| \prm - \prm' \| \\[.2cm]
& \big\| \bm{\upsilon}_{\prm'}^\top {\rm Diag}( \grd_j \bm{\Pi}_{\prm'} ) \tilde{\bm Q}_{\prm'}^i - \bm{\upsilon}_{\prm}^\top {\rm Diag}( \grd_j \bm{\Pi}_{\prm} ) \tilde{\bm Q}_{\prm}^i \big\| \leq C_{2,2} \big( i \rho^i \big) \| \prm - \prm' \| \\[.2cm]
& \big\| {\bm Q}_{\prm}^t - {\bm Q}_{\prm'}^t \big\| \leq C_1 \| \prm - \prm' \| \eqsp.
\end{split}
\eeq
The above shows that the three terms in the right hand side of \eqref{eq:upper_bd_lemma83} are proportional to $(1 + \| g \|) \| \prm - \prm' \|$ and decay geometrically with $t$ at the rate $\max\{ \rho, \lambda \}$. This implies there exists $C_2'' \in [1,\infty)$, $\delta = \max\{ \rho, \lambda \} \in [0,1)$ such that
\beq
\begin{split}
& \Big\| \PPX{t}{} \HX{\prm}{x} - h(\prm) - \Big\{ P_{\prm'}^t \HX{\prm'}{x} - h(\prm') \Big\} \Big\|  \leq C_2'' \big( t \delta^t \big) (1 + \| g \|) \| \prm - \prm' \| \eqsp.
\end{split}
\eeq
Setting $C_2 = \max\{ C_2', C_2'' \}$ concludes the proof of the current lemma.
\end{proof}

\subsection{Proof of Proposition~\ref{prop:bdd_grad}}

\begin{Prop*}
Under A\ref{ass:bdd_feat}, it holds for any $(\prm, \prm') \in \Prm^2$, $(s,a) \in \Sset \times \Aset$,
\beq
\| \grd \log \Policy_{\prm}( a ; s) \| \leq 2 \overline{b},~~\| \grd \log \Policy_{\prm}( a ; s) - \grd \log \Policy_{\prm'}( a ; s) \| \leq 8 \overline{b}^2 \| \prm - \prm' \| \eqsp.
\eeq
\end{Prop*}
\begin{proof}
To simplify notations, let us define $\Delta {\bm x} (a,b) \eqdef {\bm x}(s,a) - {\bm x}(s,b)$ as the difference between two features.
The proof is straightforward as we observe that
\beq
\begin{split}
 \grd \log \Policy_{\prm}( a ; s) = \frac{ 1 }{ \sum_{a' \in \Aset} \exp \big( \pscal{\prm}{ \Delta {\bm x} (a',a) } \big) } \sum_{b \in \Aset} \exp \big( \pscal{\prm}{ \Delta {\bm x} (b,a) } \big) \Delta {\bm x} (a,b) \eqsp.
\end{split}
\eeq
Observe that
\beq
\| \grd \log \Policy_{\prm}(a;s)  \| \leq \max_{a,b \in \Aset} \| {\bm x}(s,a) - {\bm x}(s,b) \| \leq 2 \overline{b} \eqsp.
\eeq
Moreover, the Hessian of the log policy can be evaluated as:
\beq
\begin{split}
& \grd^2 \log \Policy_{\prm}( a ; s) = \\
& \frac{ 1 }{ \sum_{a' \in \Aset} \exp \big( \pscal{\prm}{ \Delta {\bm x} (a',a) } \big) } \sum_{b \in \Aset} \exp \big( \pscal{\prm}{ \Delta {\bm x} (b,a) } \big) \Delta {\bm x} (a,b) \Delta {\bm x} (b,a)^\top - \\
&  \big( \sum_{b \in \Aset} \frac{ \exp \big( \pscal{\prm}{ \Delta {\bm x} (b,a) } \big) }{ \sum_{a' \in \Aset} \exp \big( \pscal{\prm}{ \Delta {\bm x} (a',a) } \big) }  \Delta {\bm x} (a,b) \big) \big( \frac{ \exp \big( \pscal{\prm}{ \Delta {\bm x} (b,a) } \big) }{ \sum_{a' \in \Aset} \exp \big( \pscal{\prm}{ \Delta {\bm x} (a',a) } \big) } \Delta {\bm x} (a,b) \big)^\top \eqsp.
\end{split}
\eeq
It can be checked that
\beq
\| \grd^2 \log \Policy_{\prm}( a ; s) \| \leq \max_{a,b \in \Aset } \big\| \Delta {\bm x} (a,b) \Delta {\bm x} (b,a)^\top \big\| + \big( \max_{a,b \in \Aset} \| \Delta {\bm x}(a,b) \| \big)^2 \leq 8 \overline{b}^2 \eqsp.
\eeq
This implies smoothness condition in \eqref{eq:bdd_policy_grad}.
\end{proof}

\subsection{Proof of Proposition~\ref{prop:a4a5}}
\begin{Prop*} 
Under A\ref{ass:bdd_feat}, A\ref{ass:markov}, the function
\beq \textstyle
\hHX{\prm}{x} = \sum_{t=0}^\infty \big\{ \PPX{t}{} \HX{\prm}{x} - h(\prm) \big\} \eqsp,
\eeq
is  well defined and satisfies the Poisson equation \eqref{eq:poisson-equation}.
For all $\state \in \Xset$, $(\prm, \prm') \in \Prm^2$, there exists constants $L_{PH}^{(0)}$, $L_{PH}^{(1)}$ such that
\beq
\max\{ \| \PX{} \hHX{\prm}{x} \|, \| \hHX{\prm}{x} \| \} \leq L_{PH}^{(0)},~~ \| \PX{} \hHX{\prm}{\state} - {P}_{\prm'} \hHX{\prm'}{\state} \Big\|  \leq  L_{PH}^{(1)} \| \prm - \prm' \| \eqsp.
\end{equation}
Moreover, the constants are in the order of $L_{PH}^{(0)} = {\cal O}( \frac{1}{1- \max\{\rho,\lambda\}} )$, $L_{PH}^{(1)} = {\cal O}( \frac{1}{1- \max\{\rho,\lambda\}} )$.
\end{Prop*}

\begin{proof}
From Lemma~\ref{lem:lemma83}, there exists $C_2 \in [1,\infty)$, $\delta \in [0,1)$ such that
\beq
\| \PPX{t}{} \HX{\prm}{x} - h(\prm) \| \leq C_2 \!~ t \!~ \delta^t (1 + \| g \| ),~\forall~t \geq 1,~\forall~x \in \Xset \eqsp,
\eeq
where we recall that $\delta = \max\{\rho, \lambda\}$. 
It follows that the solution to the Poisson equation $\hHX{\prm}{x}$ in \eqref{eq:poisson_def_policy} is well defined.

Moreover, it satisfies \eqref{eq:poisson-equation} and
\beq
\max\{ \| \hHX{\prm}{x} \| , \| \PPX{}{} \hHX{\prm}{x} \| \} \leq L_{PH}^{(0)} \eqsp,
\eeq
for some $L_{PH}^{(0)} = {\cal O}( \frac{1}{1- \max\{\rho,\lambda\}} ) < \infty$ (note that $g$ is bounded as specified by the state space $\Xset$). As such, the first equation in \eqref{eq:lipschitz_policy} of the proposition is proven.
Finally,
applying the definition of $\hHX{\prm}{x}$ shows that
\beq
\PPX{}{} \hHX{\prm}{x} - \PPX{}{} \hHX{\prm'}{x} = \sum_{t=1}^\infty \Big\{ \big( \PPX{t}{} \HX{\prm}{x} - h( \prm) \big) - \big( P_{\prm'}^t \HX{\prm'}{x} - h( \prm') \big) \Big\} \eqsp.
\eeq
Using Lemma~\ref{lem:lemma83}, this implies
\beq
\begin{split}
\| \PPX{}{} \hHX{\prm}{x} - \PPX{}{} \hHX{\prm'}{x} \| & \leq \sum_{t=1}^\infty \Big\| \big( \PPX{t}{} \HX{\prm}{x} - h( \prm) \big) - \big( P_{\prm'}^t \HX{\prm'}{x} - h( \prm') \big) \Big\| \\
& \leq \sum_{t=1}^\infty \Big\{ C_2 \big(t \delta^t \big) \big( 1 + \| g \| \big) \| \prm - \prm' \| \Big\} \eqsp.
\end{split}
\eeq
As such, there exists $L_{PH}^{(1)} = {\cal O}( \frac{1}{1- \max\{\rho,\lambda\}} )  \in [1,\infty)$ such that
\beq
\| \PPX{}{} \hHX{\prm}{x} - \PPX{}{} \hHX{\prm'}{x} \| \leq L_{PH}^{(1)} \| \prm - \prm' \| \eqsp,
\eeq
for all $x \in \Xset$. This proves the second equation in \eqref{eq:lipschitz_policy} of the proposition.
\end{proof}

\subsection{Proof of Proposition~\ref{prop:a1a2}}

\begin{Prop*}
Under A\ref{ass:bdd_feat}, A\ref{ass:markov}, the gradient $\grd J(\prm)$ is $\Reward_{\max} \!~ |{\cal S}| |{\cal A}|$-Lipschitz continuous. Moreover, for any $\prm \in \Prm$, it holds that
\beq
 (1 - \lambda)^2 \Gamma^2 + 2 \pscal{ \grd J(\prm) }{ h( \prm )} \geq \| h( \prm ) \|^2,~\| \grd J( \prm ) \| \leq \| h( \prm ) \| + (1-\lambda) \Gamma \eqsp,
\eeq
where $\Gamma \eqdef 2 \overline{b} \!~ \Reward_{\max} K_R \frac{ 1 }{ (1 - \rho)^2}$.
\end{Prop*}

\begin{proof}
The first statement is a direct application of part 1) in Lemma~\ref{lem:lemma82} which holds under A\ref{ass:bdd_feat}, A\ref{ass:markov}.
To prove the second statement, let us define the error vector as
\beq
\Delta( \prm ) \eqdef h( \prm ) - \grd J( \prm )
\eeq
Applying Lemma~\ref{lem:lemma82} shows that $\sup_{ \prm \in \Prm} \| \Delta( \prm ) \|^2 \leq \Gamma^2 (1 - \lambda)^2$. We observe that
\beq
\begin{split}
\pscal{ \grd J(\prm) }{ h( \prm )} & = \pscal{ h(\prm) - \Delta(\prm) }{ h( \prm )}
= \| h(\prm) \|^2 - \pscal{ \Delta(\prm) }{ h( \prm )} \\
& \geq \| h(\prm) \|^2 - \frac{1}{2} \big( \| h(\prm) \|^2 + \| \Delta ( \prm) \|^2 \big) \eqsp.
\end{split}
\eeq
This implies
\beq
\frac{\Gamma^2}{2} (1 - \lambda)^2 + \pscal{ \grd J(\prm) }{ h( \prm )} \geq \frac{1}{2} \| h( \prm ) \|^2 \eqsp.
\eeq
Furthermore, it is straightforward to show that
\beq
\| \grd J( \prm ) \| \leq \| h( \prm ) \| + \| \Delta ( \prm ) \| \leq \| h( \prm ) \| + \Gamma (1-\lambda) \eqsp,
\eeq
which concludes the proof.
\end{proof}

\section{Existence and regularity of the solutions of Poisson equations}
\label{sec:existence-properties-Poisson}
Consider the following   assumptions:
\begin{assumption} \label{ass:MC2}
For any $\prm, \prm' \in \rset^d$, we have $\sup_{\state \in \Xset} \| \PX{} ( \state, \cdot) - {P}_{\prm'} (\state, \cdot) \|_{\TV} \leq L_P \| \prm -\prm' \|$.
\end{assumption}
\begin{assumption} \label{ass:MC3a}
For any $\prm, \prm' \in \rset^d$, we have $\sup_{\state \in \Xset} \| \HX{\prm}{\state} - H_{ \prm' }( \state ) \| \leq L_H \| \prm - \prm' \|$.
\end{assumption}
\begin{assumption} \label{ass:MC1}
There exists $\rho < 1$, $K_P < \infty$ such that
\beq
\sup_{ \prm \in \rset^d , \state \in \Xset } \| \PX[n]{} ( \state, \cdot ) - \pi_{\prm} (\cdot) \|_{\TV} \leq  \rho^n  K_P,
\eeq
\end{assumption}
\begin{Lemma}
\label{lem:properties-poisson}
Assume A\ref{ass:MC2}--\ref{ass:MC1}. Then, for any $\prm \in \Prm$ and $\state \in \Xset$,
\begin{align}
\label{eq:px_bd0-poisson}
\| \hHX{\prm}{\state} \|  &\leq \frac{ \sigma K_P }{1 - \rho} \eqsp,\\
\label{eq:px_bd-poisson}
\| \PX{} \hHX{\prm}{\state} \| &\leq \frac{ \sigma \rho  K_P }{1 - \rho} \eqsp.
\end{align}
Moreover, for $\prm, \prm' \in \Prm$ and $\state \in \Xset$,
\beq \label{eq:bound_a2-poisson}
\| \PX{} \hHX{\prm}{\state} - {P}_{\prm'} \hHX{\prm'}{\state} \Big\|  \leq  L_{PH}^{(1)} \| \prm - \prm' \| \eqsp,
\eeq
where
\beq
L_{PH}^{(1)}= \frac{K_P^2 \sigma L_P}{(1-\rho)^2} \big( 2+K_P \big) + \frac{K_P}{1-\rho}  L_H \eqsp.
\eeq
\end{Lemma}
\begin{proof}
Note that, under  A\ref{ass:MC1},
\beq
\label{eq:key-properties-poisson}
\begin{split}
& \sum_{i=0}^\infty \Big\| \PX[i]{}{} (\HX{\prm}{\state} - h( \prm )) - \pi_{\prm} \big( \HX{\prm}{\cdot} - h( \prm ) \big)  \Big\| \\ 
& \leq \| \HX{\prm}{\cdot} - h( \prm ) \|_\infty \; K_P \sum_{i=0}^\infty \rho^i \leq \frac{ \sigma K_P }{1 - \rho}\eqsp.
\end{split}
\eeq
Therefore, for all $\prm \in \Prm$ and $\state \in \Xset$, the series
\beq
\sum_{i=0}^\infty  \PX[i]{}{} (\HX{\prm}{\state} - h( \prm )) - \pi_{\prm} \big( \HX{\prm}{\cdot} - h( \prm ) \big)
\eeq
is uniformly converging and is a solution of the  Poisson equation \eqref{eq:poisson-equation}.
In addition, \eqref{eq:px_bd0-poisson} and \eqref{eq:px_bd-poisson} follow directly from \eqref{eq:key-properties-poisson}.
Under A\ref{ass:MC1}, applying a simple modification\footnote{We note that under A\ref{ass:MC1}, the constants $\rho_{\theta}, \rho_{\theta'}$ are the same in \citep[Lemma 4.2]{fort2011convergence} which simplifies the derivation and yields a tighter bound.} of \citep[Lemma 4.2, 1st statement]{fort2011convergence} shows\footnote{Note that we take the measurable function as $V = 1$ therein.} that for any $\prm, \prm' \in \Prm$, we have
\beq \label{eq:pi_diff}
\| \pi_{\prm} - \pi_{\prm'} \|_{\TV} \leq \frac{K_P (1 + K_P )}{1-\rho}~
\sup_{\state \in \Xset} \| \PX{} ( \state, \cdot) - {P}_{\prm'} (\state, \cdot) \|_{\TV} \eqsp.
\eeq
Again using a simple modification of \citep[Lemma 4.2, 2nd statement]{fort2011convergence} shows that for any $\State \in \Xset$, $\prm, \prm' \in \rset^d$, it holds
\begin{align}
& \Big\| \PX{} \hHX{\prm}{\state} - {P}_{\prm'} \hHX{\prm'}{\state} \Big\| \\
& \notag \leq \frac{K_P^2}{(1-\rho)^2} \Big( \sup_{ \prm \in \Prm, \state \in \Xset } \| \HX{\prm}{\state} - {h(\prm)}\| \Big) \Big( \sup_{\state \in \Xset} \| \PX{}
( \state, \cdot) - {P}_{\prm'} (\state, \cdot) \|_{\TV} \Big) \\
& \notag \quad+ \frac{K_P}{1-\rho} \Big( \sup_{ \prm \in \Prm, \state \in \Xset } \| \HX{\prm}{\state} -{h(\prm)} \| \Big) \| \pi_{\prm} - \pi_{\prm'} \|_{\TV} + \frac{K_P}{1-\rho} \sup_{\state \in \Xset} \| \HX{\prm}{\state} - H_{ \prm' }( \state ) \| \\[.1cm]
& \notag \leq \left( \frac{K_P^2 {\sigma} L_P}{(1-\rho)^2} \big( 2+K_P \big) + \frac{K_P}{1-\rho}  L_H \right) \| \prm - \prm' \| = L_{PH}^{(1)} \| \prm - \prm' \| \eqsp,
\end{align}
where the last inequality is due to A\ref{ass:MC2}, A\ref{ass:MC3a}, A\ref{ass:MC3} and \eqref{eq:pi_diff}.
\end{proof}

\end{document}